\documentclass[letters]{fcs}

\usepackage{layout}
\usepackage{hyperref}
\usepackage{thmtools}
\usepackage{mathtools}
\usepackage{graphicx}
\usepackage{subcaption}
\usepackage{placeins}    

\theoremstyle{thmstyleone}

\theoremstyle{thmstyletwo}%
\theoremstyle{thmstylethree}%

\newcommand{\pomdp}{\mathcal{N}}

\newcommand{\real}{\mathbb{R}}

\newcommand{\pr}{P}

\title{Synthetic POMDPs to Challenge Memory-Augmented RL: Memory Demand Structure Modeling}

\author[1]{Yongyi~WANG}
\author[1]{Lingfeng~LI}
\author[1]{Bozhou~CHEN}
\author[1]{Ang~LI}
\author[1]{Hanyu~LIU}
\author[1]{Qirui~ZHENG}
\author[1]{Xionghui~YANG}
\author[1,+]{Wenxin~LI}
\address[1]{School of Computer Science, Peking University, Beijing 100871, China}
\corremail{lwx@pku.edu.cn}
\fcssetup{
  received = {January 16, 2026},
  accepted = {April 8, 2026},
}
\begin{document}
\section{Introduction}
Partially Observable Markov Decision Processes (POMDPs) provide a canonical framework for reinforcement learning (RL) under partial observability, where memory-augmented RL algorithms such as PPO \cite{schulman2017proximal} with LSTM \cite{hochreiter1997long} compress trajectories to infer hidden states. While both real-world and synthetic POMDP benchmarks exist \cite{cobbe2020leveraging,hafner2021benchmarking,morad2023popgym,tao2025benchmarking,shchendrigin2026memory}, synthetic environments offer controllability for targeted evaluation.

Current synthetic POMDPs are either built from scratch or derived from MDPs via masking, noise, or reward redistribution. However, two issues remain: the lack of a theoretical foundation for designing memory demands, and the inadequacy of existing memory difficulty adjustments: masking neglects information value difference, and scaling alters the belief MDP. Consequently, performance degradation cannot be attributed solely to increasing memory demands.

To address these limitations, we introduce a principled framework for synthesizing POMDPs that systematically vary memory demands while preserving the underlying belief MDP, ensuring that difficulty stems solely from memory demands.

Our main contributions are:

1. \textbf{A principled theoretical framework} to analyze POMDPs via Memory Demand Structure (MDS) and transition invariance. MDS defines the minimal history required for optimal decisions; invariance properties (stationarity, consistency) formalize how dynamics generalize across different temporal steps and trajectories.

2. \textbf{A general construction methodology} using autoregressive processes, state aggregation, and reward delay to design POMDPs with predefined MDS. Autoregressive processes enable direct construction; state aggregation and reward delay derive POMDPs from MDPs while preserving optimal policies.

3. \textbf{A suite of lightweight, scalable, fine-grained diagnostic POMDP environments} with progressively increasing memory demand that complements existing benchmarks, which
enables targeted evaluation of memory models’ capabilities including handling different levels of temporal dependencies, non-stationarity and non-consistency.

Technical details and data are given in Online Resource 1.

\section{Theory of Memory Demand Structure Modeling}
To characterize what POMDPs demand from memory-augmented RL, we introduce the MDS framework, which identifies the minimal set of past moments an agent must retain to predict the next observation and reward. We leverage MDS to design synthetic environments.
\begin{definition}[\textbf{MDS}] For any POMDP $\pomdp$'s trajectory $h_t=(z_{0:t},a_{0:t},r_{0:t-1})$, a set $D\subseteq\{0,1,\cdots,t\}$ is called $h_t$'s \textbf{MDS} if $\forall z_{t+1}\in Z,~\forall r_t\in\real$, \label{def:mds}
$$\pr(z_{t+1},r_t|z_{0:t},a_{0:t})\!=\!\pr(z_{t+1},r_t|\bigcap_{\tau\in D}\{Z_\tau\!=\!z_\tau, A_\tau\!=\!a_\tau\})$$
\end{definition}
Given that memory models, particularly RNNs, tend to exhibit recency bias and limited capacity, we select the MDS based on the "fewest time steps" or "temporally nearest" criterion. This allows us to establish a lower bound on the POMDP memory difficulty that aligns with these inherent characteristics of memory models.
However, while MDS specifies which historical moments a transition in a given POMDP trajectory depends on, it does not quantify the amount of information contributed by observations at those moments. To address this, we introduce the concept of Memory Complexity Polynomial (MCP) as a quantitative tool for measuring the informational contribution of observations at each MDS-identified moment to the transition.
\begin{definition}[\textbf{MCP}]
The MCP of a POMDP $\pomdp$ trajectory $h_t$ is
$$d_{h_t}(x)\coloneqq\frac{1}{|D|!}\!\!\sum_{\sigma}\!\sum_{\tau\in D}I((z,a)_\tau;(z_{t+1},r_t)|(z,a)_{\{\tau^\prime:\sigma(\tau^\prime)<\sigma(\tau)\}})x^{t-\tau}\label{eq:mcpdef}$$
where $I(\cdot)$ is mutual information, and $\sigma$ traverse $D$'s all permutations.
\end{definition}
MCP models moments in the MDS as players in a cooperative game, where moment $\tau\in D$ corresponds to the term $x^{t-\tau}$. With a Shapley value-like form and the chain rule of mutual information, it distributes the total information of a single-step transition across these moments.
Although potentially broadly applicable, MCP's high computational cost limits its use primarily to synthetic POMDP analysis and construction. Extending it to complex real-world environments requires efficiency gains via simplifying assumptions, sampling approximations, and parallelized estimation of mutual information and Shapley values.

MDS and MCP capture the temporal dependencies along POMDP trajectories, but whether the transition dynamics is stationary over time or consistent across trajectories remains uncaptured. To address this, we introduce the concept of transition invariance, which characterizes the generalization of the transition dynamics.
\begin{definition}[\textbf{POMDP's Transition Invariance}]
Given a POMDP $\pomdp$ and an equivalence relation $\simeq$ on its trajectory set $H$, $\pomdp$ is called \textbf{$\simeq$-invariant} if for any $h_{t_1},h^\prime_{t_2}\in H$ such that $h_{t_1}\simeq h^\prime_{t_2}$, the following holds: $\forall z\in Z,~\forall r\in\real,~P(z,r|z_{0:t_1},a_{0:t_1})=P(z,r|z_{0:t_2}^\prime,a_{0:t_2}^\prime)$.
\label{def:hdpinvariance}
\end{definition}
\begin{table}[h]
\centering
\begin{tabular}{|c|c|c|}
\hline
\textbf{Invariance} & \textbf{Stationary} & \textbf{Non-stationary} \\
\hline
\textbf{Consistent} & $\simeq_k$ & $\simeq_k\wedge\simeq^n$ \\
\hline
\textbf{Non-consistent} & $\simeq_k\wedge\simeq^{\lfloor mz_0\rfloor}$ & $\simeq_k\wedge\simeq^n\wedge\simeq^{\lfloor mz_0\rfloor}$ \\
\hline
\end{tabular}
\caption{Example of different transition invariances. $\simeq_k$ denotes a relation where trajectories $h_{t_1},h_{t_2}^\prime$ are equivalent if $(z,a)_{t_1-k+1:t_1}=(z^\prime,a^\prime)_{t_2-k+1:t_2}$; $\simeq^n$ where equivalent if $\lfloor t_1/n\rfloor=\lfloor t_2/n\rfloor$; $\simeq^{\lfloor mz_0\rfloor}$ where equivalent if $\exists i,z_0,z_0^\prime\in[\frac{i}{m},\frac{i+1}{m})$.}
\label{tab:hdps}
\end{table}
\section{Synthetic POMDP Environments}
\label{sec:syntheticenvs}
\subsection{Constructed: Linear \& Nonlinear Autoregressive Processes}
\label{subsec:linproc}
To construct a POMDP that satisfies a given MDS, the transition dynamics can be directly defined via an autoregressive mechanism. 
Motivated by the prevalence of linear systems in real-world control, we focus on POMDPs whose MDS mirrors that of a $k$-th order MDP, that is, for any history $h_t$, $z_{t+1},r_t$ depends only on $z_{t-k+1:t}$.
Within this MDS family, we implement the dynamics using autoregressive processes: the first $k$ observations are drawn from $U(0,1)$, and each subsequent is a linear combination of the previous $k$ ones (linear) or their square (nonlinear), taken modulo $1$ to ensure boundedness.
The coefficients are selected according to the four transition invariance types in Def. \ref{tab:hdps}, with distinct values assigned to non-equivalent transition steps.
The reward mechanism gives a reward of $1$ for correctly selecting the interval that contains the next observation.

Specifically, we define $k$-th order autoregressive process environments. 
AllEqOneLinProcEnv($k$) sets all coefficients to $1$, while \{All, Time, Traj, No\}EqLinProcEnv($k$) use distinct coefficients matching the transition invariance types in Table \ref{tab:hdps} (C\&S, N\&S, C\&N, N\&N).
AllEqOneQuadProcEnv($k$) replaces AllEqOneLinProcEnv($k$)'s linear sum  with a sum of squares, introducing nonlinearity.
\subsection{Derived: Convolving MDP States}
\label{subsec:mdpconv}
Modifying existing MDPs to endow them with a specified MDS leverages the richness of available environments and avoids the oversimplified or idealized constructions that come with manually designed dynamics.
Inspired by echo and denoising in signal processing, we construct POMDPs by applying invertible convolution to existing MDP states, which provably preserves the optimal policy. Specifically, at each step, the observation is generated by convolving the previous $k$ MDP states with a fixed kernel. Once the original MDP state contains no decision-irrelevant information, this setup forces the memory model to learn a deconvolution that recovers the underlying states. 

Specifically, StateConv\_\{n,p\}(k) is a family of wrappers that convolve MDP states with coefficients: $1, \pm(1-2^{-k})$ (with $+$ for \_p and $-$ for \_n). The resulting MDS spans all time steps, and each state's contribution to decoding the current observation decays exponentially with temporal distance.

\subsection{Derived: Delaying MDP Rewards}
\label{subsec:mdpdelay}
State convolution yields non-Markovian transitions, while reward delaying produce non-Markovian rewards, which are also common in real-world settings.
Although reward delay preserves the Markovian optimal policy in theory, delayed rewards in practice impede learning, requiring memory models to discard irrelevant past states. Thus, reward delaying and state aggregation offer complementary tests: the former evaluates forgetting, the latter retention.

We define RewardDelay$(k)$ wrappers: for any MDP trajectory, the reward at step $t-k$ is delayed to $t$, yielding an MDS of $\{t-k, t\}$. At the terminal step, all rewards from the last $k$ steps are collected.

\subsection{Basic and Wrapped MDP Environments}
We normalized classic MDP by fixing trajectory length and normalizing episode return.
For example, NormalizedCartPole has trajectory length $256$ and episode return range $[-1,1]$.

We also constructed special MDPs to be wrapped, such as RewardWhenInside, with observations independently sampled from $U(0,1)$ for $256$ steps. Agents receive a reward of $1$ for correctly identifying the interval (out of $8$) that the current state falls into, otherwise $0$.

We denote StateConv-wrapped environments as \_p$(k)$ or \_n$(k)$ (e.g., RewardWhenInside\_p$(1)$), and RewardDelay-wrapped environments as $(k)$ (e.g., NormalizedCartPole$(k)$).

\subsection{Environment Selection According to Real-world Memory Demand}
LinProcEnv (low/high order) for different order short-term dependence, StateConv for long-term dependence, and RewardDelay for evaluating forgetting.
TimeEq for non-stationarity, TrajEq for non-consistency.

\section{Experiments to Validate MDS-Guided POMDP Design}
\label{sec:experiments}
We take RL algorithm and memory model implementations from POPGym \cite{morad2023popgym} to validate our MDS-based framework and synthetic environment suites (introduced in Section \ref{sec:syntheticenvs}): increasing MDS complexity consistently and uniformly degrades convergence performance.
Each setting was run for training $10^6$ steps with $10$ random seeds and mean episode return with $95\%$ confidence intervals was reported.

Fig.~\ref{fig:alleqoneproc} shows that for both PPO \cite{schulman2017proximal} and DQN \cite{mnih2013playing}, increasing the order of autoregressive process consistently degrades the convergence performance across various memory models. In environments with linear transitions (Figs.~\ref{fig:alleqonelinproc}, \ref{fig:alleqonequadproc}), simpler RNN performs comparably to or marginally better than LSTM; whereas in nonlinear settings (particularly Fig.~\ref{fig:alleqonequadprocdqn}), LSTM exhibits a clear advantage.

Fig.~\ref{fig:alllinproc} shows that performance degrades with increasing order. For both PPO and DQN, non-consistency harms performance more than non-stationarity, likely because enforcing non-consistency via initial-state distinction introduces dependence on the initial observation, which increases the MDS complexity.

Fig.~\ref{fig:convenvs} shows that higher order StateConv introduce more interference from past observations, making state decoding harder. Performance degrades monotonically with level in RewardWhenInside, where states are compact and task-relevant. A similar but weaker trend appears in NormalizedCartPole. The spike at Level~5 in Fig.~\ref{fig:cartpole_n} suggests that using $s_t - s_{t-1}$ as observation facilitates learning more than $s_t - \mu s_{t-1}, \mu\in(0,1)$, likely due to representational redundancy in CartPole.
Thus, compact representation learning \cite{zhang2020learning,lan2023sample} facilitates adapting StateConv to complicated real-world environments.

Fig.~\ref{fig:rewarddelay} shows performance degrading with longer reward delays. As reward delays preserve the Markovian optimal policy, memory models do not outperform MLP. Thus, RewardDelay tests a model's ability to discard unnecessary information, complementing StateConv. In wrapped NormalizedCartPole (Fig.~\ref{fig:normalizedcartpole}), the monotonic decline disappears. This is possibly because CartPole, as a continuous control task, is inherently robust to reward delays. Under our setup, PPO with memory models still learns a near-optimal Markov policy.

\begin{figure}[h]
\captionsetup[subfigure]{
    font=scriptsize,
    justification=raggedright,
    skip=0pt,
    margin=0pt,
    aboveskip=0pt,
    belowskip=0pt
}
\centering
    \begin{subfigure}[b]{0.48\linewidth}
        \centering
        \includegraphics[width=\linewidth, keepaspectratio, height=0.6\linewidth]{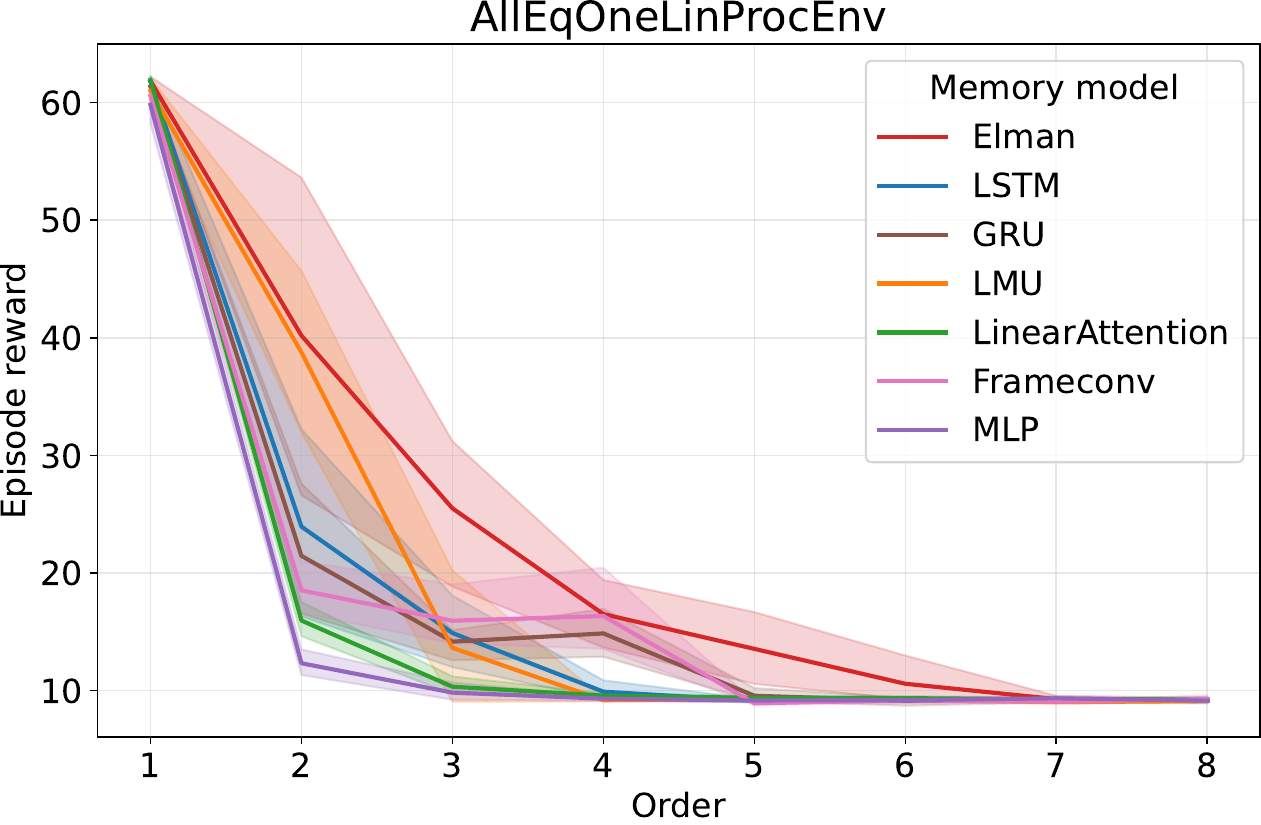}
        \caption{(PPO) AllEqOneLinProcEnv$(k)$}
        \label{fig:alleqonelinproc}
    \end{subfigure}
    \hfill
    \begin{subfigure}[b]{0.48\linewidth}
        \centering
        \includegraphics[width=\linewidth, keepaspectratio, height=0.6\linewidth]{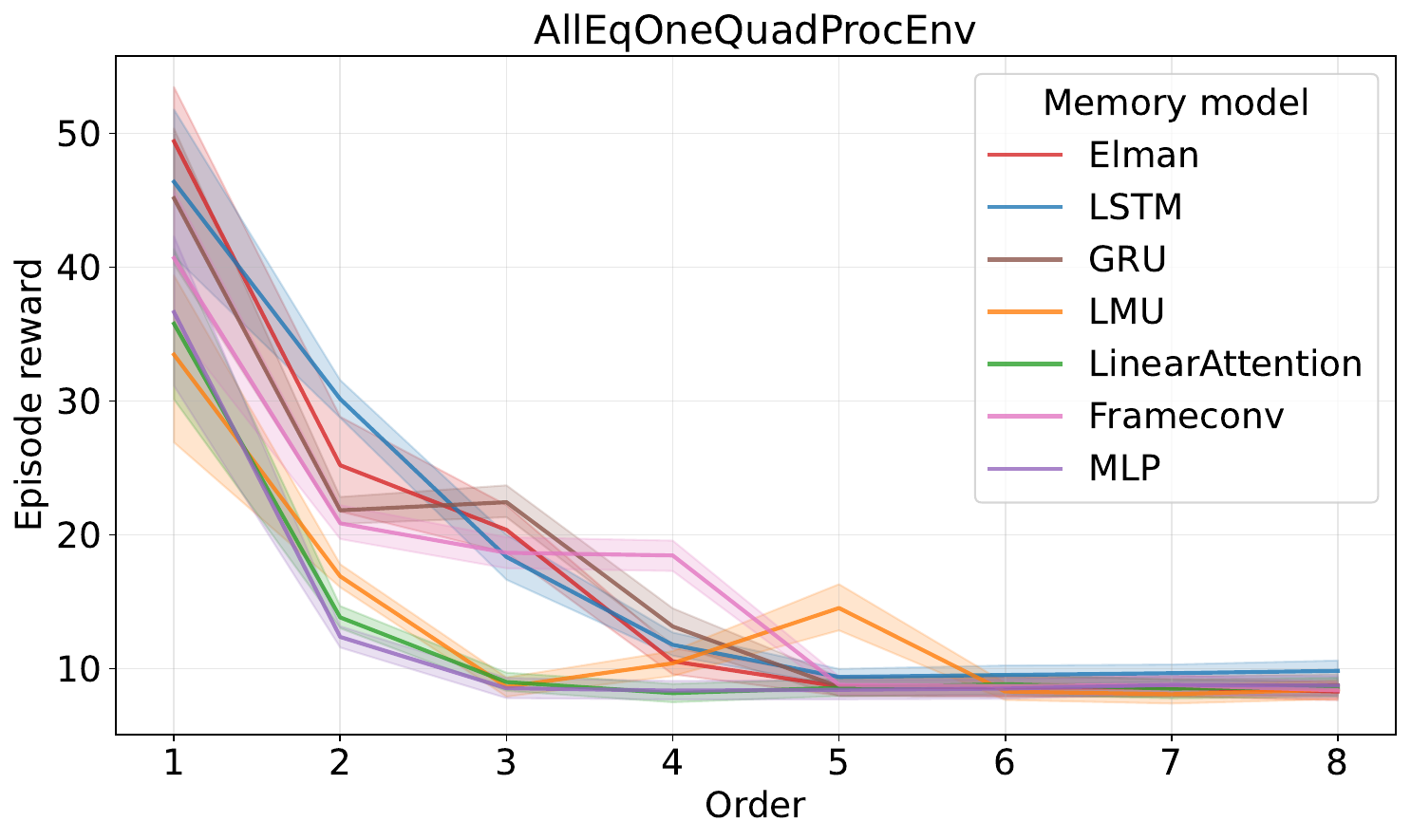}
        \caption{(PPO) AllEqOneQuadProcEnv$(k)$}
        \label{fig:alleqonequadproc}
    \end{subfigure}
    \vspace{0pt}
    \begin{subfigure}[b]{0.48\linewidth}
        \centering
        \includegraphics[width=\linewidth, keepaspectratio, height=0.6\linewidth]{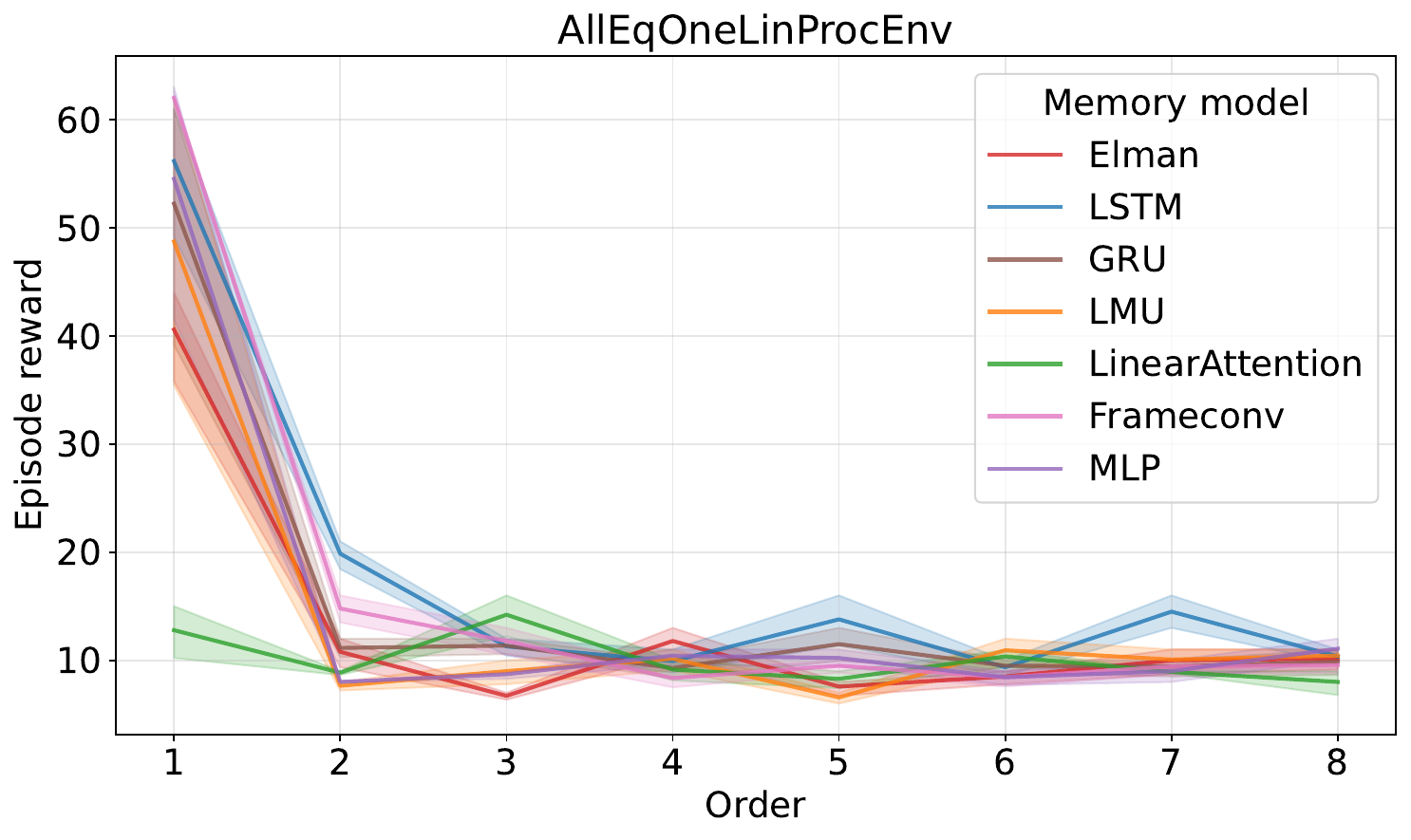}
        \caption{(DQN) AllEqOneLinProcEnv$(k)$}
        \label{fig:alleqonelinprocdqn}
    \end{subfigure}
    \hfill
    \begin{subfigure}[b]{0.48\linewidth}
        \centering
        \includegraphics[width=\linewidth, keepaspectratio, height=0.6\linewidth]{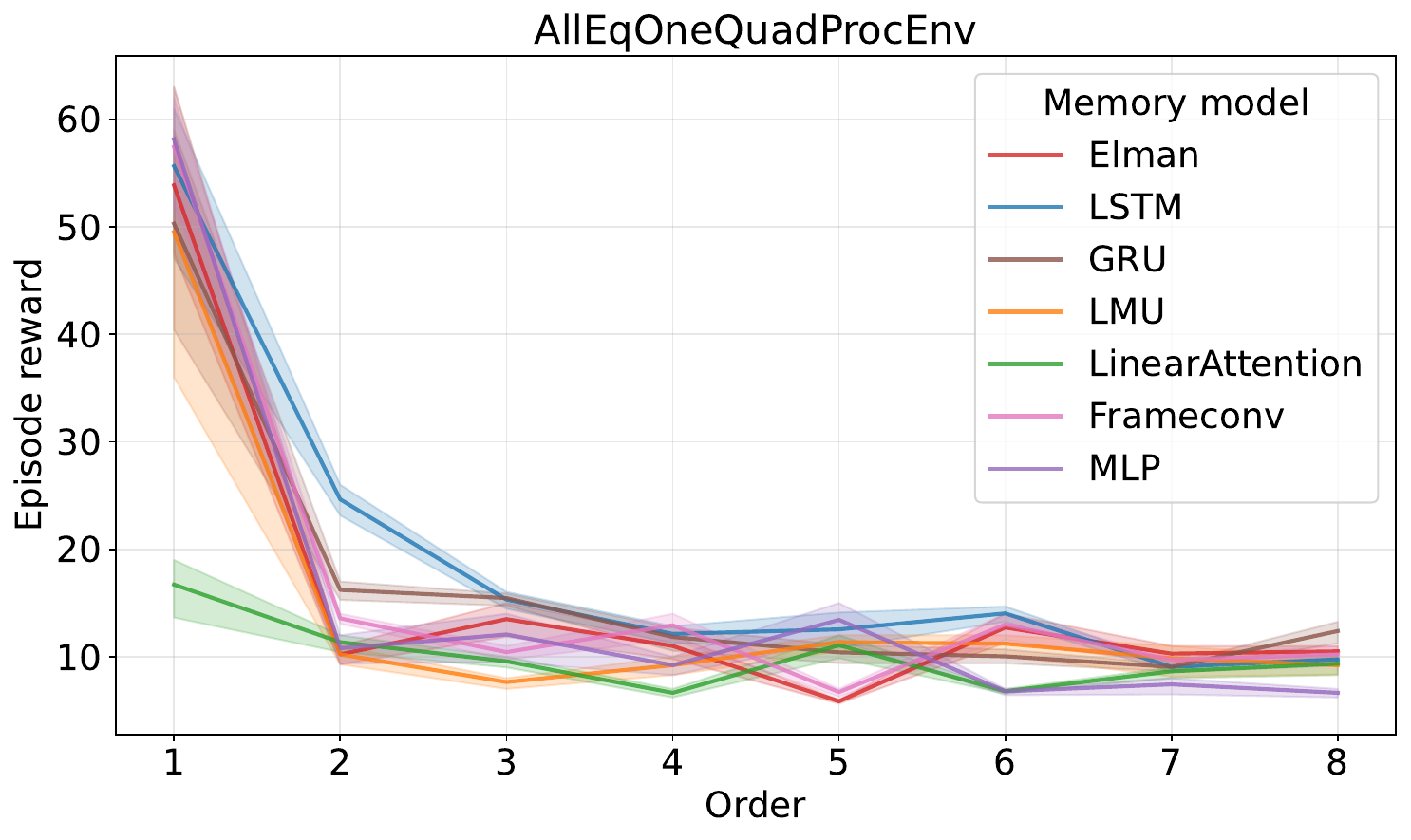}
        \caption{(DQN) AllEqOneQuadProcEnv$(k)$}
        \label{fig:alleqonequadprocdqn}
    \end{subfigure}
\caption{Linear/Non-linear autoregressive processes prediction tasks with all $1$ coefficients.}
\label{fig:alleqoneproc}
\end{figure}

\begin{figure}[h]
    \captionsetup[subfigure]{
        font=scriptsize,
        justification=raggedright,
        skip=0pt,
        margin=0pt,
        aboveskip=0pt,
        belowskip=0pt
    }
    \centering
    \begin{subfigure}[b]{0.48\linewidth}
        \centering
        \includegraphics[width=\linewidth, keepaspectratio, height=0.6\linewidth]{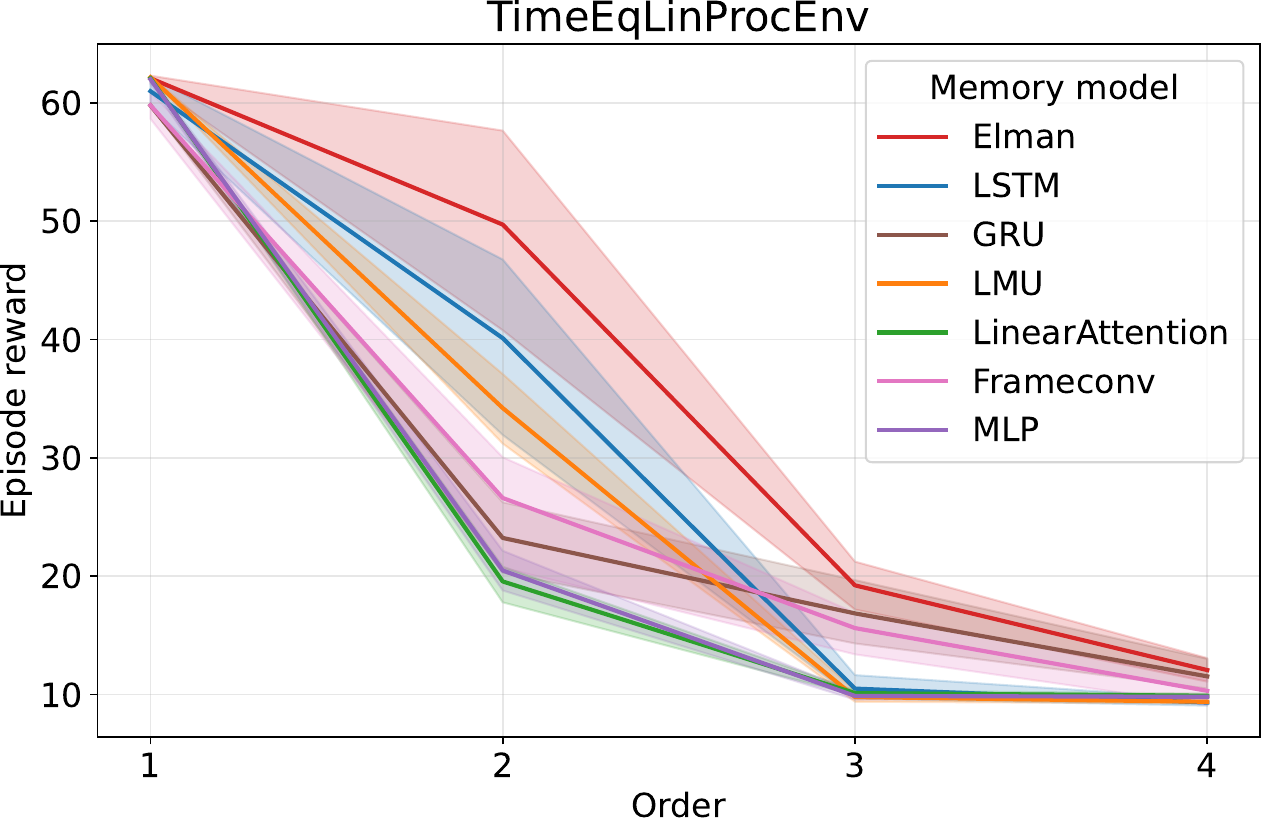}
        \caption{(PPO) Consistent, Non-stationary}
        \label{fig:timeeqproc}
    \end{subfigure}
    \vspace{0pt}
    \begin{subfigure}[b]{0.48\linewidth}
        \centering
        \includegraphics[width=\linewidth, keepaspectratio, height=0.6\linewidth]{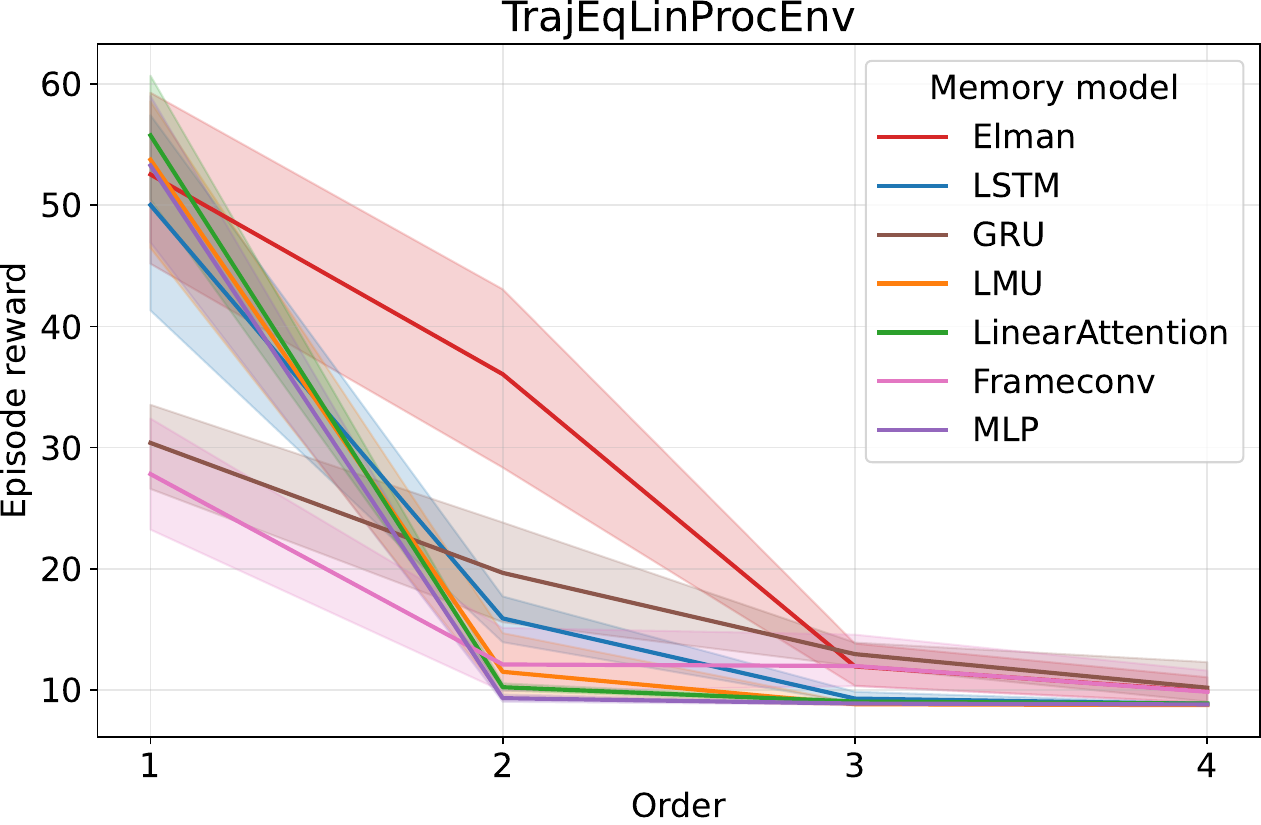}
        \caption{(PPO) Non-consistent, Stationary}
        \label{fig:trajeqproc}
    \end{subfigure}
    \hfill
    \begin{subfigure}[b]{0.48\linewidth}
        \centering
        \includegraphics[width=\linewidth, keepaspectratio, height=0.6\linewidth]{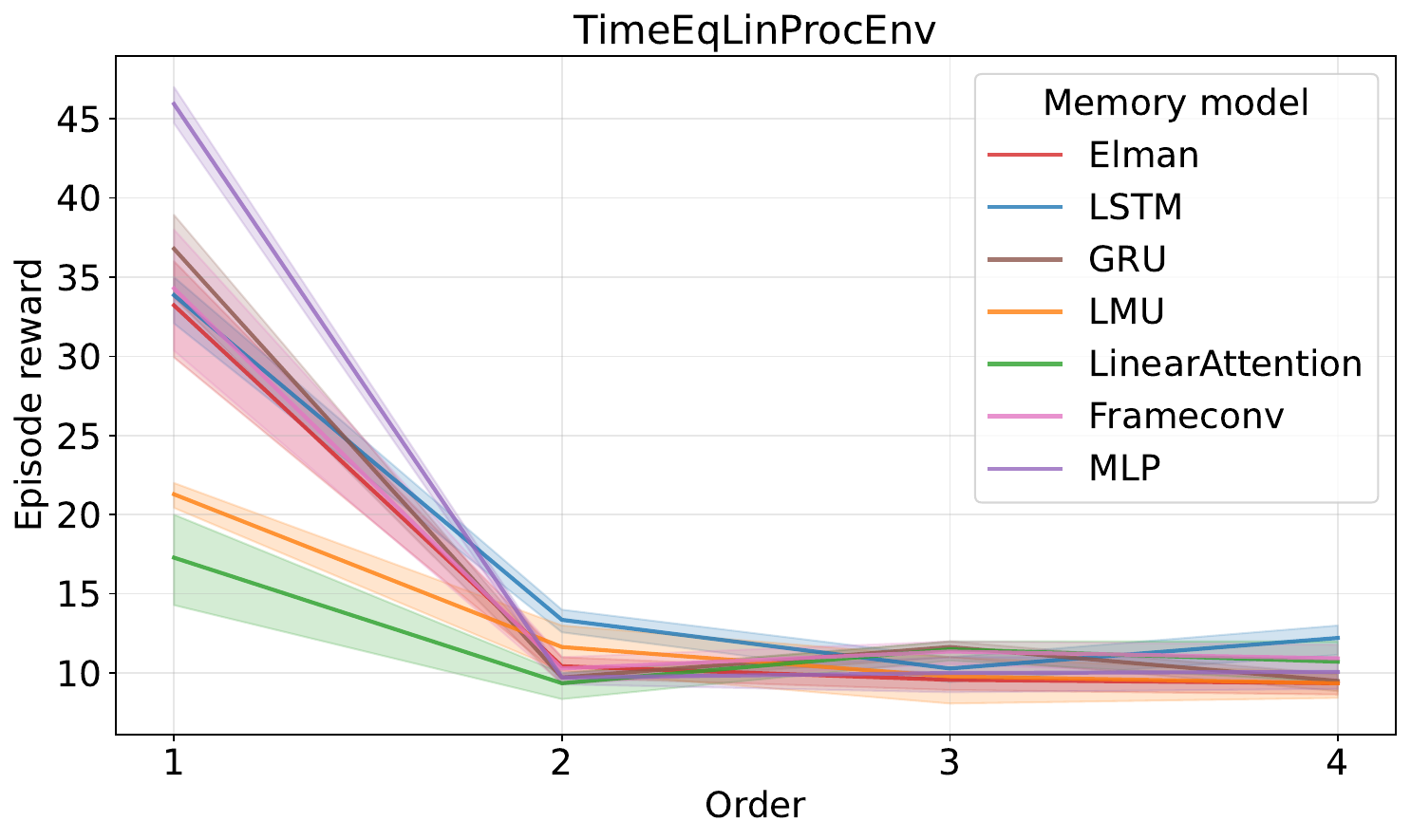}
        \caption{(DQN) Consistent, Non-stationary}
        \label{fig:timeeqprocdqn}
    \end{subfigure}
    \vspace{0pt}
    \begin{subfigure}[b]{0.48\linewidth}
        \centering
        \includegraphics[width=\linewidth, keepaspectratio, height=0.6\linewidth]{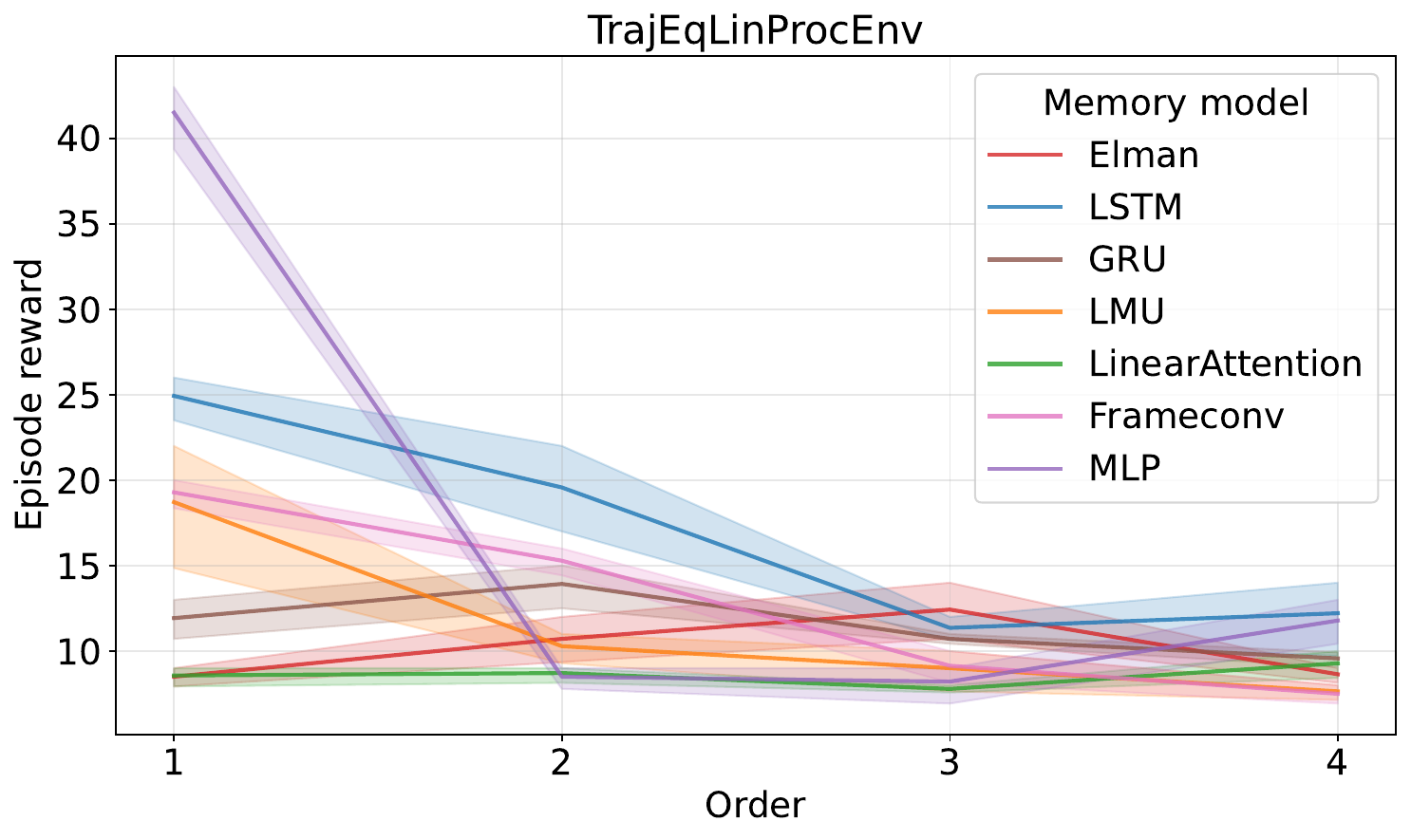}
        \caption{(DQN) Non-consistent, Stationary}
        \label{fig:trajeqprocdqn}
    \end{subfigure}
    \hfill
    \caption{Linear processes with different coefficients and transition invariances.}
    \label{fig:alllinproc}
\end{figure}

\begin{figure}[h]
    \captionsetup[subfigure]{
        font=scriptsize,
        justification=raggedright,
        skip=1pt,
        margin=0pt,
        aboveskip=0pt,
        belowskip=0pt
    }
    \centering
    \begin{subfigure}[b]{0.48\linewidth}
        \centering
        \includegraphics[width=\linewidth, keepaspectratio, height=0.6\linewidth]{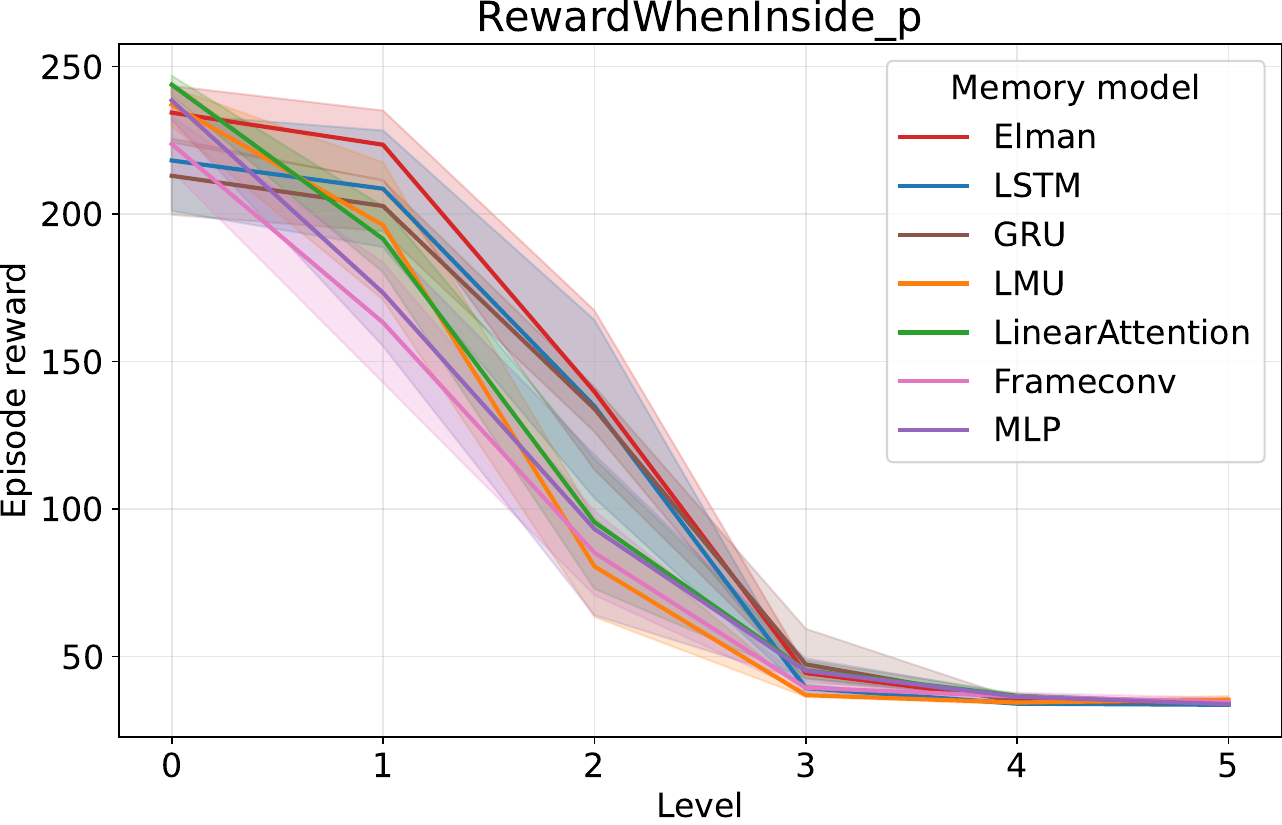}
        \caption{RewardWhenInside\_p}
        \label{fig:r_p}
    \end{subfigure}
    \hfill
    \begin{subfigure}[b]{0.48\linewidth}
        \centering
        \includegraphics[width=\linewidth, keepaspectratio, height=0.6\linewidth]{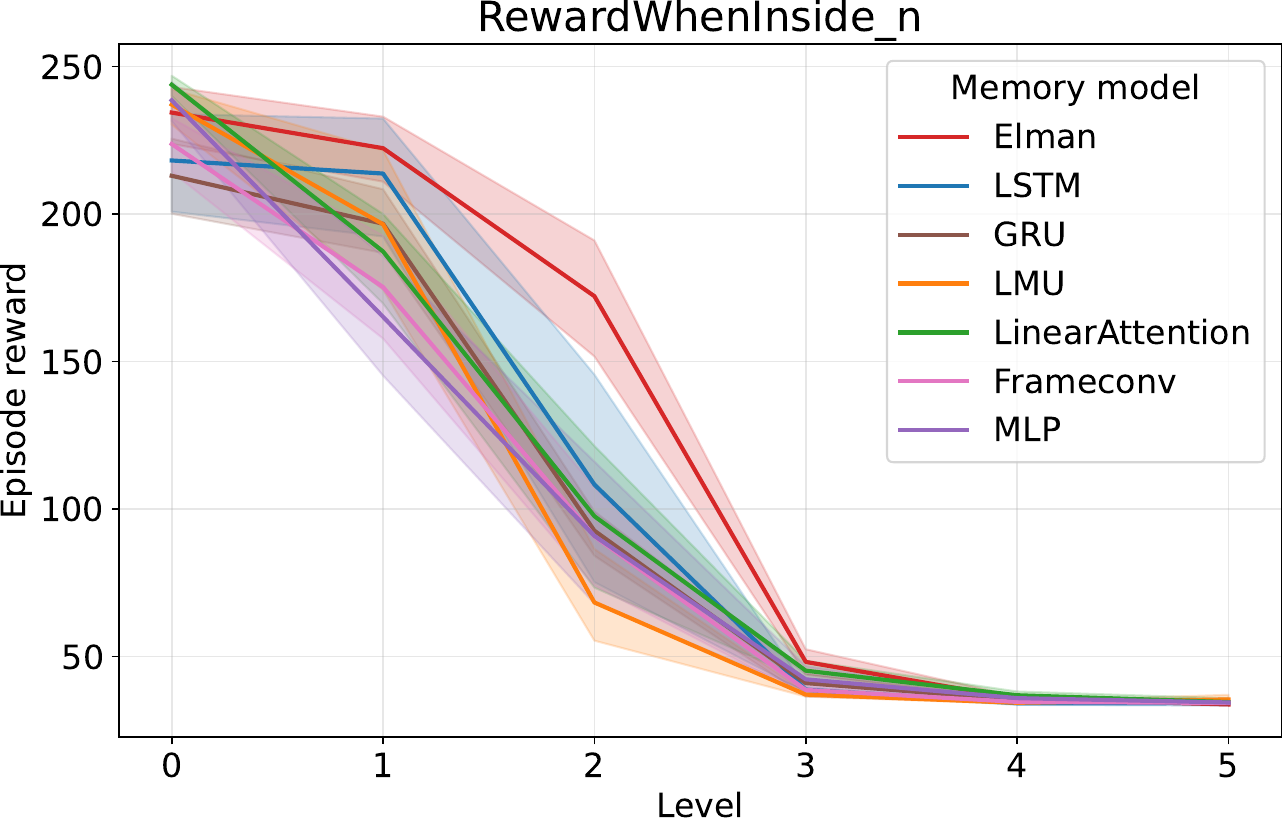}
        \caption{RewardWhenInside\_n}
        \label{fig:r_n}
    \end{subfigure}
    \vspace{0pt}
    \begin{subfigure}[b]{0.48\linewidth}
        \centering
        \includegraphics[width=\linewidth, keepaspectratio, height=0.6\linewidth, height=0.6\linewidth]{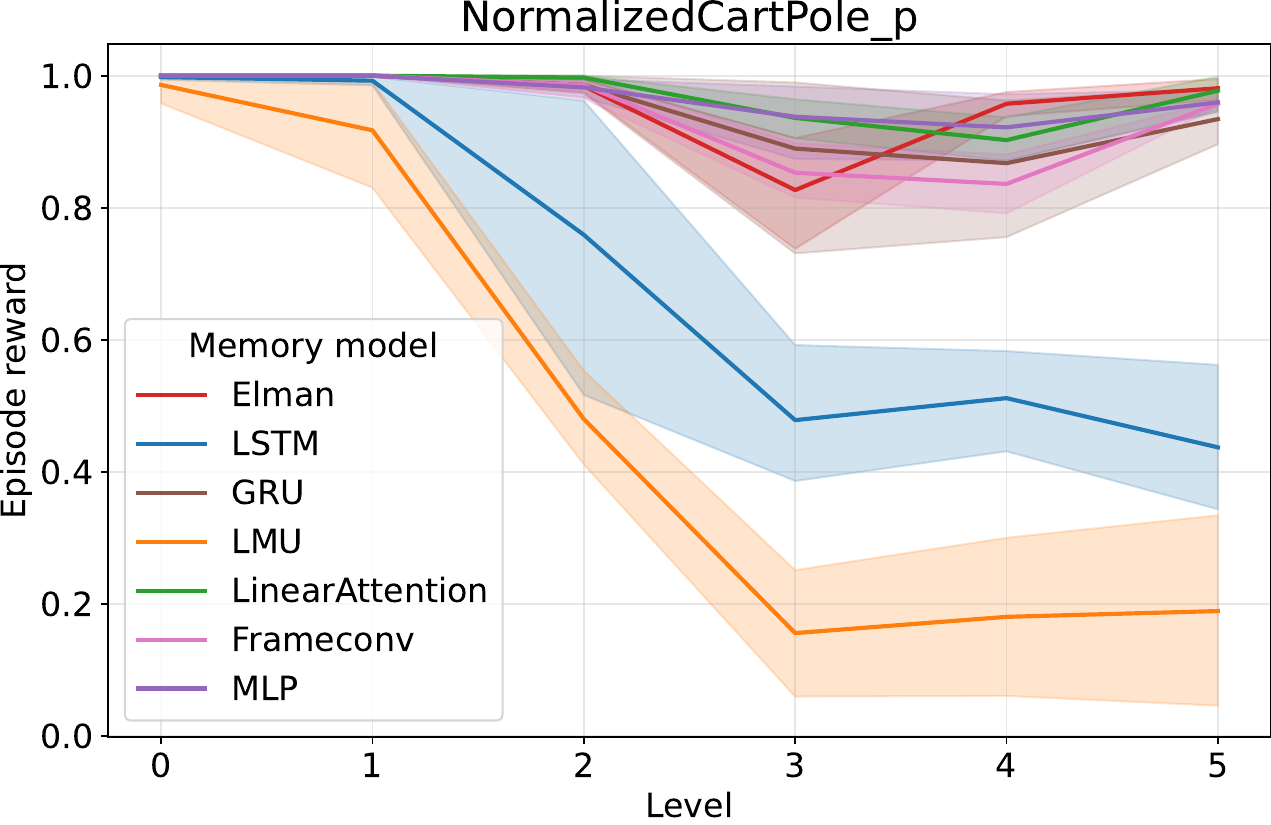}
        \caption{NormalizedCartPole\_p}
        \label{fig:cartpole_p}
    \end{subfigure}
    \hfill
    \begin{subfigure}[b]{0.48\linewidth}
        \centering
        \includegraphics[width=\linewidth, keepaspectratio, height=0.6\linewidth]{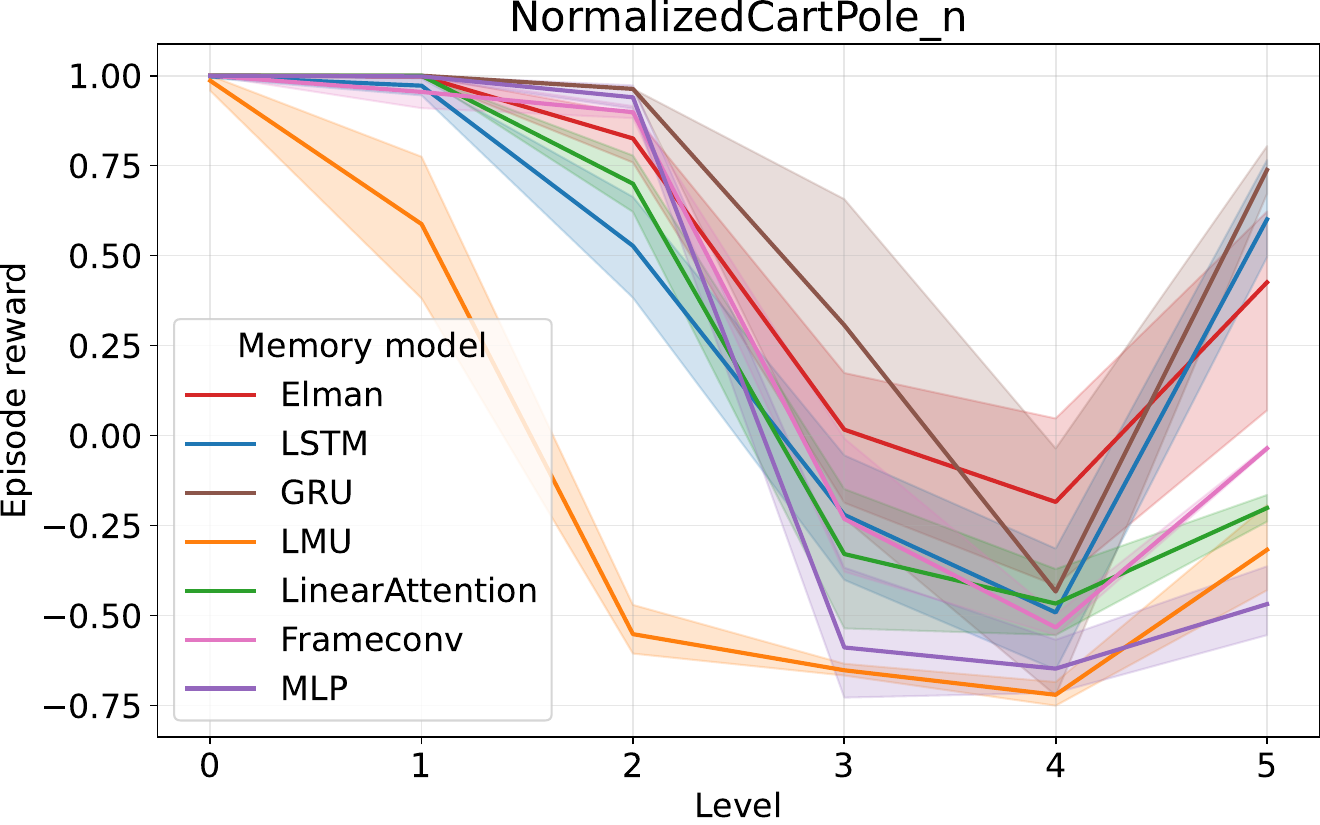}
        \caption{NormalizedCartPole\_n}
        \label{fig:cartpole_n}
    \end{subfigure}
    \vspace{0pt}
    \caption{(PPO) StateConv series with increasing coefficient $w_1$.}
    \label{fig:convenvs}
\end{figure}

\begin{figure}[h]
    \captionsetup[subfigure]{
        font=scriptsize,
        justification=raggedright,
        skip=1pt,
        margin=0pt,
        aboveskip=0pt,
        belowskip=0pt
    }
    \centering
    \begin{subfigure}[b]{0.48\linewidth}
        \centering
        \includegraphics[width=\columnwidth]{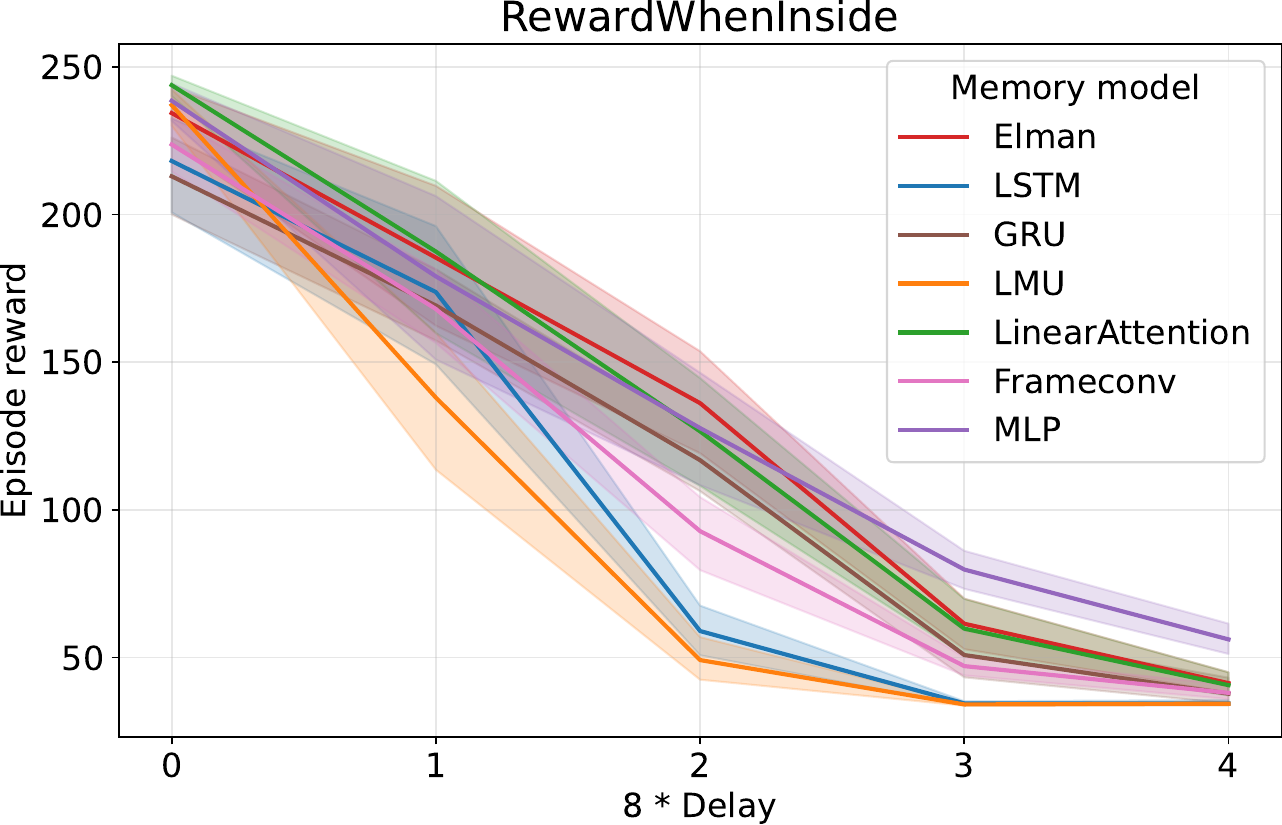}
    \caption{RewardWhenInside}
        \label{fig:rewarddelay}
    \end{subfigure}
    \hfill
    \begin{subfigure}[b]{0.48\columnwidth}
        \includegraphics[width=\columnwidth]{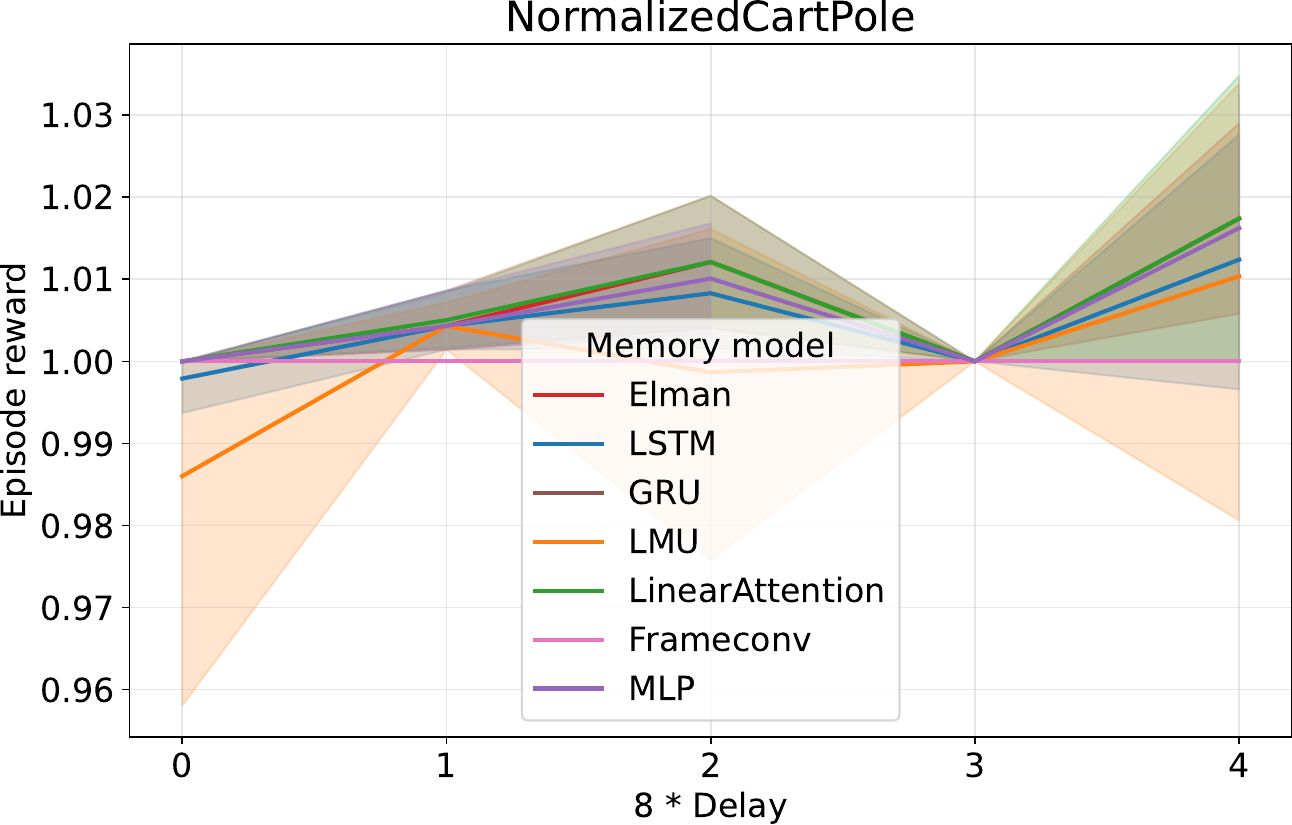}
        \caption{NormalizedCartPole}
        \label{fig:normalizedcartpole}
    \end{subfigure}
    \caption{(PPO) RewardDelay series with increasing delays.}
    \label{fig:rd}
\end{figure}

\section*{Competing interests}
The authors declare that they have no competing interests or financial
conflicts to disclose.

\bibliographystyle{fcs}
\bibliography{ref}

\end{document}


\maketitle
\appendix

\section{POMDP's Observation Transition Function}
\label{app:proofeqpomdp}
In our work, we directly construct some transition functions between observations, which deviates from the standard POMDP formulation (where transitions occur between states, and observations are generated from unobserved actions via an observation function). Such observation-to-observation transitions can always be reduced to a POMDP by treating the entire history as a hidden state. However, whether an arbitrary POMDP can be equivalently expressed using observation-based transition functions remains to be proven.

Below, we show the way to express the transition dynamics of an arbitrary POMDP as observation-to-observation transition functions. Based on this conclusion, and with a slight abuse of notation, we will make more use of observation-based transition functions rather than state-based ones in the following text.

\begin{theorem}[\textbf{POMDP's Observation Transition Function}]\label{thm:eqpomdp}
For any POMDP $\pomdp=\pomdpdef$, it's observation transition function is:
\begin{equation*}
T_t^\prime(z_{t+1},r_t | z_{0:t},a_{0:t}) = \frac{\sum\limits_{s_{0:t+1}\in S^{t+2}}T_t(s_{t+1},r_t|s_t,a_t)O_{t+1}(z_{t+1}|s_{t+1})\Phi_t(z_{0:t}|s_{0:t},a_{0:t})}{\sum\limits_{s_{0:t}\in S^{t+1}}\Phi_t(z_{0:t}|s_{0:t},a_{0:t})}
\end{equation*} 
where \begin{equation*}
\Phi_t(z_{0:t}|s_{0:t},a_{0:t})=\rho_0(s_0)O_0(z_0|s_0)\prod_{\tau=0}^{t-1}O_{\tau+1}(z_{\tau+1}|s_{\tau+1})\sum_{r_\tau\in\real}T_\tau(s_{\tau+1},r_\tau|s_{\tau},a_{\tau})
\end{equation*}
\end{theorem}
\begin{proof}
\begin{flalign*}
& T_t^\prime(z_{t+1},r_t | z_{0:t}, a_{0:t}) \\
& = \frac{\pr(Z_{0:t+1}=z_{0:t+1}, R_t=r_t | A_{0:t}=a_{0:t})}{\pr(Z_{0:t}=z_{0:t}|A_{0:t}=a_{0:t})} \\
& = \frac{\sum\limits_{s_{0:t+1}\in S^{t+2}}\pr(s_0)\pr(z_0|s_0)\pr(s_{t+1},r_t|s_t,a_t)\pr(z_{t+1}|s_{t+1})\prod_{\tau=0}^{t-1}\pr(s_{\tau+1}|s_{\tau},a_{\tau})\pr(z_\tau|s_\tau)}{\sum\limits_{s_{0:t}\in S^{t+1}}\pr(s_0)\pr(z_0|s_0)\prod_{\tau=0}^{t-1}\pr({s_{\tau+1}|s_{\tau},a_{\tau})\pr(z_\tau|s_\tau)}} \\
& = \frac{\sum\limits_{s_{0:t+1}\in S^{t+2}}\rho_0(s_0)O_0(z_0|s_0)T_t(s_{t+1},r_t|s_t,a_t)O_{t+1}(z_{t+1}|s_{t+1})\prod_{\tau=0}^{t-1}O_{\tau+1}(z_{\tau+1}|s_{\tau+1})\sum\limits_{r_\tau\in\real}T_\tau(s_{\tau+1},r_\tau|s_{\tau},a_{\tau})}{\sum\limits_{s_{0:t}\in S^{t+1}}\rho_0(s_0)O_0(z_0|s_0)\prod_{\tau=0}^{t-1}O_{\tau+1}(z_{\tau+1}|s_{\tau+1})\sum\limits_{r_\tau\in\real}T_\tau(s_{\tau+1},r_\tau|s_{\tau},a_{\tau})} \\
& = \frac{\sum\limits_{s_{0:t+1}\in S^{t+2}}T_t(s_{t+1},r_t|s_t,a_t)O_{t+1}(z_{t+1}|s_{t+1})\Phi_t(z_{0:t}|s_{0:t},a_{0:t})}{\sum\limits_{s_{0:t}\in S^{t+1}}\Phi_t(z_{0:t}|s_{0:t},a_{0:t})} 
\end{flalign*}
\end{proof}

\section{Futher Discussions of Memory Demand Structure (MDS)}
\label{app:fdmds}
\subsection{More Refined Characterization of MDS}
\label{app:mdsdefandextend}
\begin{definition}[\textbf{POMDP MDS}] For any POMDP $\pomdp$'s trajectory $h_t=(z_{0:t},a_{0:t},r_{0:t-1})$, a set $D\subseteq\{0,1,\cdots,t\}$ is called $h_t$'s \textbf{MDS} if $\forall z_{t+1}\in Z,~\forall r_t\in\real$, \label{def:mds}
$$\pr(z_{t+1},r_t|z_{0:t},a_{0:t})\!=\!\pr(z_{t+1},r_t|\bigcap_{\tau\in D}\{Z_\tau\!=\!z_\tau, A_\tau\!=\!a_\tau\})$$
\end{definition}
We denote $\mathbf{D}(h_t)$, the set of all MDSs of $h_t$; $D(h_t)\coloneqq \{(z_\tau,a_\tau)\}_{\tau\in D}$, which is a sufficient statistic; and $\tilde{D}(h)\coloneqq \bigcap_{\tau\in D}\{Z_\tau=z_\tau, A_\tau=a_\tau\}$, a sufficient condition to estimate the next observation $z_{t+1}$ and reward $r_t$. \\

\textbf{Note}: The results below are self-evident.
\begin{itemize}
\item MDP is a special case of POMDP, where $Z=S, \ O_t(Z_t=s_t|S_t=s_t)=1$.
\item For any trajectory $h_t$ of a POMDP, we have $\{0,1,\dots,t\}\in\mathbf{D}(h_t)$.
\item For any trajectory $h_t$ of an MDP, we have $\{t\}\in\mathbf{D}(h_t)$.
\item For any $D\in\mathbf{D}(h_t)$ and $X\subseteq\{0,1,\dots,t\}$, we have $D\cup X\in\mathbf{D}(h_t)$.
\end{itemize}

For a trajectory $h_t$ of an POMDP, $D(h_t)$ is defined to capture the relationship between the next observation-reward pair and past observation-action pairs.
This relationship can be decomposed into four components: the relation between the next observation and past observations, between the next observation and past actions, between the next reward and past observations, and between the next reward and past actions.
This decomposition can be achieved by taking the marginal distributions of $z_{t+1}$ and $r_t$ from $T_t(z_{t+1},r_t|z_{0:t},a_{0:t})$ along with the technique of separating the observation and action components in $\tilde{D}(h_t)$.

Such a discussion enables a more refined characterization of MDS.
When we focus on the associations between the next-step observation-reward pair and previous observations within the MDS, our emphasis lies in the \textbf{memory} of the past by the memory model, specifically, which observations provided by the environment have served as clues in the process of reaching the current status.
Conversely, when our focus shifts to the associations between the next-step observation-reward pair and previous actions within the MDS, our emphasis turns to the \textbf{exploration} behavior of the RL algorithm: that is, what sequence of actions is required to reach the current status. 

In the other dimension, we can define Markov/non-Markov state transition dynamics when focusing solely on the relationship between next observation and past observation-action pairs.
Conversely, we can also define Markov/non-Markov rewards when concentrating only on the relationship between the next reward and prior observation-action pairs.

Consequently, POMDPs can be categorized into: Markov-transition-Markov-reward, Markov-transition-non-Markov-reward, Non-Markov-transition-Markov-reward, and Non-Markov-transition-non-Markov-reward.

\subsection{Example of an MDP Trajectory with Multiple MDS}
\label{app:egmdp}
\begin{example}
\begin{figure}[h]
\centering
\includegraphics[width=0.5\columnwidth]{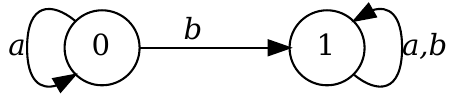}
\caption{A deterministic MDP with all $0$ reward: $S=\{0,1\},~A=\{a,b\},~\rho_0(0)=1$}
\label{fig:multiplemds}
\end{figure}
In MDP Fig. \ref{fig:multiplemds} (as a special case of POMDP), given the trajectory $h_2=(011,bbb,00)$, then $\mathbf{D}(h_2)=2^{\{0,1,2\}}$.
\end{example}
\begin{proof}
Clearly, the union of multiple MDSs of the same trajectory remains an MDS of that trajectory. Therefore, we only need to show that $\{0\}, \{1\}, \{2\} \in \mathbf{D}(h_2)$. 

Since given $s_2=1$ and $1$ is an absorbing state, it is obvious that $s_3$ will always be $1$, $\{2\} \in \mathbf{D}(h_2)$. 

Given that $s_1=1$, then $s_2=1$ and $s_3=1$ can be determined sequentially, making $\{1\}\in \mathbf{D}(h_2)$. 

Moreover, since $s_0=0$ and $a_0=b$ imply $s_1=1$, then $\{0\}\in\mathbf{D}(h_2)$ as well. 
\end{proof}
\subsection{MDS of Example POMDP with Visualization}
\begin{center}
\begin{figure}[h]
\label{app:explainmds}
\centering
\includegraphics[width=0.8\columnwidth]{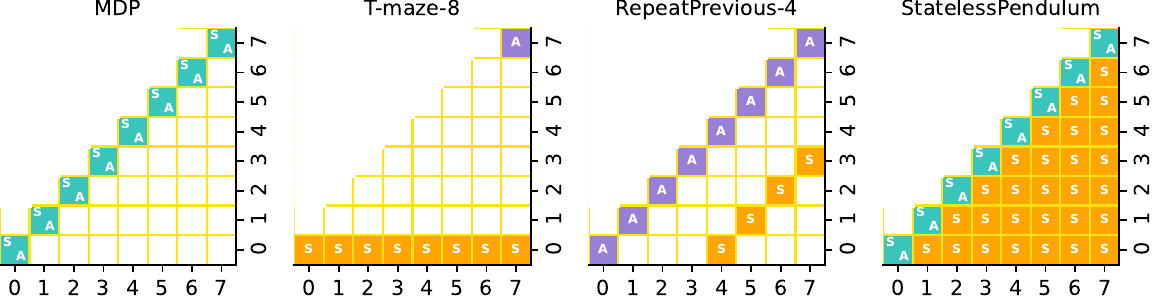}
\caption{MDS visualizations. A colored block $(t,i)$ indicates that for any trajectory $h_t$, $i\in D\in\mathbf{D}(h_t)$. Different colors denote distinct dependence relationships: state-only, action-only, and both.}
\label{fig:egmds}
\end{figure}
\end{center}

\textbf{MDP}: In a general \textit{MDP}, the next state-reward joint distribution depends solely on the current state and action.

\textbf{T-Maze-8}: consists of an 8-length one-way channel with a T-junction at one end. Agents start from the end without the T-junction and receive instructions on the correct direction at the outset. To gain reward (the sole means of doing so), they must remember this direction while traversing the channel and select it accurately upon reaching the T-junction.
Under this mechanism, each observation encountered while traversing the tunnel carries information from the initial instruction (information embedded in the initial observation), and whether a reward is obtained at the 8th step depends on the action chosen at that final moment.

\textbf{RepeatPrevious-4}: samples observations via an iid process. Agents must repeat the observation observed 4 time-steps prior.
Under this mechanism, the subsequent reward hinges on both the current action and the observation from 4 steps ago.

\textbf{StatelessPendulum}: The classical \textit{Pendulum} can be modeled as an MDP, where the angular position and angular velocity define its observation.
The \textit{StatelessPendulum} variant removes the angular position from the original \textit{Pendulum}'s observation representation.
As a result, agents retain a memory of all historical observations of angular velocities to reconstruct the original angular position through aggregation.

\subsection{Further Discussions of Memory Complexity Polynomial (MCP)}
\label{app:mcp}
\begin{definition}[\textbf{MCP}]
The MCP of a POMDP $\pomdp$ trajectory $h_t$ is
\begin{equation}d_{h_t}(x)\coloneqq\frac{1}{|D|!}\!\!\sum_{\sigma}\!\sum_{\tau\in D}I((z,a)_\tau;(z_{t+1},r_t)|(z,a)_{\{\tau^\prime:\sigma(\tau^\prime)<\sigma(\tau)\}})x^{t-\tau}\label{eq:mcpdef}\end{equation}
where $I(\cdot)$ is mutual information, and $\sigma$ traverse $D$'s all permutations.
\end{definition}
The MCP is a polynomial designed to model the memory demand imposed on memory models by a specific trajectory $h_t$ in a POMDP. 
The coefficient of $x^k$ quantifies the average amount of information that the observation-action pair $(z_{t-k},a_{t-k})$ at time step $t-k$ provides for predicting the next observation transition $(z_{t+1},r_t)$.

The term "average amount of information" is adopted rather than simply "amount of information" to account for variations in coefficient assignments that may arise from different sequential decompositions of mutual information, as illustrated in the following example.
\begin{figure}[h]
\centering
\includegraphics[width=0.6\columnwidth]{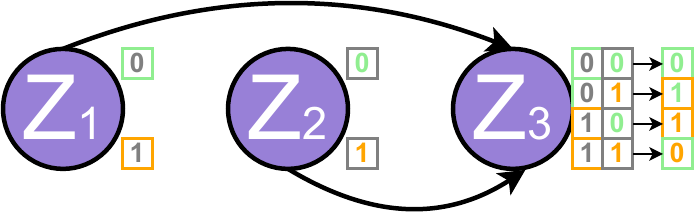}
\caption{ Two random coins determine the third random variable.}\label{fig:coin}
\end{figure}
\begin{example}
Consider two coins. Let independent random variables $Z_1$ and $Z_2$ (Fig. \ref{fig:coin}) represent the outcomes of random tosses of these two coins, respectively, where $0$ and $1$ denote heads and tails. Define the random variable $Z_3 \coloneqq (Z_1 + Z_2) \mod 2$. 
Then, the mutual information $I(Z_3;Z_1) = I(Z_3;Z_2) = 0$ and $I(Z_3;Z_2|Z_1)=I(Z_3;Z_1|Z_2)=1$.
\label{eg:coin}
\end{example}
\begin{proof}
$I(Z_3;Z_1)=H(Z_3)+H(Z_1)-H(Z_3,Z_1)=1+1-2=0$.
Similarly, $I(Z_3,Z_2)=0$. \\
Since $I(Z_3;Z_2,Z_1)=H(Z_3)+H(Z_2,Z_1)-H(Z_3,Z_2,Z_1)=1+2-2=1$,
by the chain rule of mutual information, $I(Z_3;Z_2,Z_1)=I(Z_3;Z_1)+I(Z_3;Z_2|Z_1)=I(Z_3;Z_2|Z_1)$, we have $I(Z_3;Z_2|Z_1)=1$.
Similarly, $I(Z_3;Z_1|Z_2)=1$.
\end{proof}
Consider Eg. \ref{eg:coin}, we have $$\underbrace{I(Z_3;Z_2,Z_1)}_{1\text{ bit}}=\underbrace{I(Z_3;Z_2)}_{0\text{ bit, coefficient of } x^0}+\underbrace{I(Z_3;Z_1|Z_2)}_{1\text{ bit, coefficient of 
 } x^1}=\underbrace{I(Z_3;Z_1)}_{0\text{ bit, coefficient of } x^1}+\underbrace{I(Z_3;Z_2|Z_1)}_{1\text{ bit, coefficient of }x^0}$$
Therefore, if we consider removing the averaging over $\sum_\sigma$ in Eq. \ref{eq:mcpdef}, for some decomposition order $\sigma$, it becomes:
\begin{equation*}
d_{h_t}(x)=\sum_{\tau\in D}I((z,a)_\tau;(z_{t+1},r_t)|(z,a)_{{\tau^\prime:\sigma(\tau^\prime)<\sigma(\tau)}})x^{t-\tau}
\end{equation*}

However, when applying the chain rule for mutual information $I((z_{t+1},r_t);(z,a)_{\tau\in D})$ with different orders of decomposition $\sigma$, the resulting coefficients are not necessarily identical. This is precisely why the definition requires averaging over all permutations $\sigma$ of the elements in set $D$.
After taking the average, the resulting definition becomes independent of the ordering in the chain rule decomposition of mutual information, thus effectively capturing the amount of information that a observation-action pair at a given moment can provide for predicting the next transition. Taking the average over all permutations in this way is precisely analogous to the calculation of the Shapley value in cooperative games.
\subsection{MCP Example}
\label{app:mcpeg}
The definition of MCP provides a way to compute, for a single trajectory (viewed as a sequence of random events), the contribution of information at each time step to the one-step transition. When analyzing the entire POMDP, it is necessary to take the expectation of the MCP over trajectories with respect to some probability measure on the set of trajectories, where the uniform distribution is chosen as the default.
\\\\
\textbf{T-maze-8}: Given that every step is required to retain $1$ bit of information ($2$ actions, Up or Down, $\log 2=1$ bit) originating from time step $0$, as demonstrated in Fig. \ref{fig:egmds}, the MCP is:
$$d_\text{\textit{T-maze-8}}(x)=(1+x+x^2+x^3+x^4+x^5+x^6+\frac{1+x^7}{2!})/8=\frac{1}{8}[\frac{1}{2}(x^7+1)+\sum_{\tau=0}^6x^\tau]$$
\textbf{RepeatPrevious-4}: Similar to the \textit{T-maze-8} environment, correctly reproducing an observation from $4$ steps earlier requires the agent to retain information amounting to $\log|Z|$, the logarithm of the observation space size. The MCP is thus given by:
$$d_\text{\textit{RepeatPrevious-4}}(x)=(\log|Z|)\times(1+1+1+1+\frac{1+x^4}{2!}+\frac{1+x^4}{2!}+\frac{1+x^4}{2!}+\frac{1+x^4}{2!})/8=\frac{1}{4}(x^4+3)\log|Z|$$
\textbf{AllEqOneLinProcEnv}$(k)$: When the trajectory length is sufficiently large compared to the order $k$, the proportion of trajectories with $t\le k$ becomes negligible. Moreover, these short trajectories correspond only to lower-order terms in the polynomial, and their omission is thus inconsequential. In this family of environments, every time step within a window of length $k$ contributes equally to the information at a given moment, since in Eq. \ref{eq:mcpdef}, for any permutation of these steps, the final one receives the full information as its monomial coefficient. Thus, the following conclusion can be drawn:
$$d_{\text{AllEqOneLinProcEnv}(k)}(x)=(\log_2 8)(1+x+x^2+\dots+x^{k-1})/k=\frac{3}{k}\sum_{\tau=0}^kx^\tau$$
\section{POMDP's Transition Invariance}
\subsection{Preliminaries}
\begin{definition}[\textbf{Independent Decision Process (IDP)}]
An \textbf{IDP} $\mathcal{I}\coloneqq\langle A, T_{0:\infty}\rangle$ consists of: action set $A$, reward dynamics at time $t$: $T_t: A\rightarrow\Delta_\real$.
\end{definition}
\textbf{Note}: IDP is a special case of MDP with only one state, which is why $\rho_0$ and $S$ are omitted from the definition.
\begin{definition}[\textbf{Stationary MDP}]
An MDP $\mdp=\mdpdef$ is \textbf{stationary} if $\exists T:S\times A\to\Delta_{S\times\real},~\forall t\in\nat,~T_t=T$.
\end{definition}
\textbf{Note}: In most of the literature, the MDP is stationary by default in the form of $\langle \rho_0, S, A, T\rangle$.
\begin{definition}[\textbf{High-Order MDP}]
A \textbf{$k$-th ($k\in\nat^+$) order MDP} $\mdp_k\coloneqq\mdpdef$ consists of: initial state distribution $\rho_0\in\Delta_S$; state set $S$; action set $A$; transition and reward dynamics at time $t$:
$$T_t:\begin{cases} S^k\times A^k \to \Delta_{S\times\real}, & (t\ge k-1) \\ S^{t+1}\times A^{t+1} \rightarrow \Delta_{S\times\real}, & (t<k-1)\end{cases}$$
\end{definition}

\textbf{Note}: The transition dynamics of $k$-th order MDP takes states and actions from the current step back up to at most $k$ previous steps as input. It can also be defined as a special case of POMDP that has a constant MDS representation $\forall h_t,~\{t,t-1,t-2,\dots,t-k+1\}\cap\nat\in\mathbf{D}(h_t)$.
\begin{definition}[\textbf{The Prefix Relation between Trajectories}]
Given a POMDP $\pomdp=\pomdpdef$ with trajectory set $H$, for any $h_{t_1},h^\prime_{t_2}\in H,~(t_1\le t_2)$, if $h_{t_1}=(z_{0:t_1},a_{0:t_1},r_{0:t_1-1}),~h^\prime_{t_2}=(z_{0:t_2},a_{0:t_2},r_{0:t_2-1})$, then $h_{t_1}$ is a \textbf{prefix} of $h^\prime_{t_2}$, denoted $h_{t_1}\preceq h_{t_2}^\prime$.
\label{def:prefix}
\end{definition}
\textbf{Note}: 
When a policy is specified, and new observations and rewards are generated step-by-step according to the transition dynamics specified by the POMDP, the resulting trajectory sequence $h_0,h_1,h_2,\dots$ satisfies $h_0\preceq h_1\preceq h_2\preceq\dots$.
\subsection{Equivalence Relations Form a Lattice}
\label{app:equivlattice}
In the main text, the symbol "$\wedge$" is used to combine equivalence relations, and this section will rationalize this way of expression.
\begin{definition}[\textbf{Lattice} (via partially ordered set (poset))]
Let $\langle L,\preceq\rangle$ be a poset, where $\preceq$ denotes a partial order on the set $L$.
$\langle L, \preceq\rangle$ is called a \textbf{lattice} if for all $a,b\in L$:
\begin{itemize}
\item The \textbf{supremum} (denoted $a\vee b$) of $a$ and $b$ exists in $L$, which satisfies $a\preceq a\vee b,~b\preceq a\vee b$, and for any $c\in L$ with $a\preceq c$ and $b\preceq c$, it follows that $a\vee b\preceq c$.
\item The \textbf{infimum} (denoted $a\wedge b$) of $a$ and $b$ exists in $L$, which satisfies $a\wedge b\preceq a,~a\wedge b\preceq b$, and for any $c\in L$ with $c\preceq a$ and $c\preceq b$, it follows that $c\preceq a\wedge b$.
\end{itemize}
\label{def:latticeposet}
\end{definition}
\begin{definition}[\textbf{Lattice} (via algebraic operations)]
Let $\langle L,\vee,\wedge\rangle$ be an algebraic system, where $\vee$ and $\wedge$ are binary operations on L.
$\langle L,\vee,\wedge\rangle$ is called a \textbf{lattice} if for all $a,b,c\in L$:
\begin{itemize}
\item \textbf{Commutativity}: $a\vee b=b\vee a$ and $a\wedge b=b\wedge a$.
\item \textbf{Associativity}: $(a\vee b)\vee c = a\vee(b\vee c)$ and $(a\wedge b)\wedge c = a\wedge(b\wedge c)$.
\item \textbf{Absorption}: $a\vee(a\wedge b)=a$ and $a\wedge(a\vee b)=a$.
\end{itemize}

\textbf{Note}: These operations induce a partial order $\preceq$ on $L$ via $a\preceq b\iff a\wedge b=a$ (or equivalently $a\vee b=b$), making the definitions above equivalent.
\end{definition}
\begin{theorem}[\textbf{Lattice of Equivalence Relations}]
Let $S$ be a set, denote $E(S)$ the set of all equivalence relations on $S$, then $\langle E(S),\subseteq\rangle$ is a lattice.
\end{theorem}
\begin{proof}
By the property of $\subseteq$, $\langle E(S),\subseteq\rangle$ forms a poset.
For any $\phi,\psi\in E(S)$, we define:
\begin{equation}\phi\wedge\psi~\coloneqq\phi\cap\psi\label{eq:latticewedge}\end{equation}
\begin{equation}\phi\vee\psi~\coloneqq\phi\cup(\psi\circ\phi)^\ast\label{eq:latticevee}\end{equation}
where $\forall x,y\in E(S)$, $x\circ y\coloneqq \{(a,b)|~\exists c\in S,~(a,c)\in x,~(c,b)\in y\}$ and $x^\ast\coloneqq x\cup x\circ x\cup x\circ x\circ x\cup \dots$. \\
First, we show $\phi\wedge\psi,\phi\vee\psi\in E(S)$.
\begin{itemize}
\item\textbf{$\phi\wedge\psi$}:
\begin{itemize}
\item\textbf{Reflexivity}: $\forall a\in S$, $(a,a)\in\phi,(a,a)\in\psi\implies(a,a)\in\phi\cap\psi=\phi\wedge\psi$.
\item\textbf{Symmetry}: $\forall a,b\in S,~(a,b)\in\phi\wedge\psi\implies(a,b)\in\phi,~(a,b)\in\psi\implies(b,a)\in\phi,~(b,a)\in\psi\implies(b,a)\in\phi\wedge\psi$.
\item\textbf{Transitivity}:
$\forall a,b,c\in S,~(a,b),(b,c)\in\phi\wedge\psi\implies(a,b),(b,c)\in\phi,~(a,b),(b,c)\in\psi\implies(a,c)\in\phi,~(a,c)\in\psi\implies(a,c)\in\phi\wedge\psi$.
\end{itemize}
\item\textbf{$\phi\vee\psi$}:
\begin{itemize}
\item\textbf{Reflexivity}:
$\forall a\in S,~(a,a)\in\phi\implies(a,a)\in\phi\cup(\psi\circ\phi)^\ast=\phi\vee\psi$.
\item\textbf{Symmetry}:
$\forall a,b\in S,~(a,b)\in\phi\vee\psi\implies(a,b)\in\phi\cup(\psi\circ\phi)^\ast\implies(a,b)\in\phi$ or $(a,b)\in(\psi\circ\phi)^\ast$.

1. If $(a,b)\in\phi$, then $(b,a)\in\phi\implies(b,a)\in\phi\vee\psi$.

2. If $(a,b)\in(\psi\circ\phi)^\ast$, then $\exists k\in\nat,~\exists x_{0:2k-1}\in S^{2k},~a=x_0,~b=x_{2k-1},~\forall i\in\nat,i<k,~(x_{2i},x_{2i+1})\in\psi,~(x_{2i+1},x_{2i+2})\in\phi\implies(x_{2i+1},x_{2i})\in\psi,~(x_{2i+2},x_{2i+1})\in\phi\implies(b,a)\in(\psi\circ\phi)^k\subseteq(\psi\circ\phi)^\ast\subseteq\phi\vee\psi$.
\item\textbf{Transitivity}:
$\forall a,b,c\in S,~(a,b),(b,c)\in\phi\vee\psi\implies(a,b),(b,c)\in\phi$ or $(a,b)\in\phi,~(b,c)\in(\psi\circ\phi)^\ast$ or $(a,b)\in(\psi\circ\phi)^\ast,~(b,c)\in\phi$ or $(a,b),(b,c)\in(\psi\circ\phi)^\ast$.

1. If $(a,b),(b,c)\in\phi$, then $(a,c)\in\phi\subseteq\phi\vee\psi$.

2. If $(a,b)\in\phi,~(b,c)\in(\psi\circ\phi)^\ast$, then $(a,c)\in\phi\circ(\psi\circ\phi)^\ast=(\{(x,x)|x\in S\}\circ\phi)\circ(\psi\circ\phi)^\ast\subseteq(\psi\circ\phi)\circ(\psi\circ\phi)^\ast=(\psi\circ\phi)^\ast\subseteq\phi\vee\psi$.

3. If $(a,b)\in(\psi\circ\phi)^\ast,~(b,c)\in\phi$, then $\exists k\in\nat,~(a,c)\in(\psi\circ\phi)^k\circ\phi=(\psi\circ\phi)^k\subseteq(\psi\circ\phi)^\ast\subseteq\phi\vee\psi$.

4. If $(a,b),(b,c)\in(\psi\circ\phi)^\ast$, then $(a,c)\in(\psi\circ\phi)^\ast\circ(\psi\circ\phi)^\ast=(\psi\circ\phi)^\ast\subseteq\phi\vee\psi$.
\end{itemize}
\end{itemize}
Then we show that such definitions of $\wedge,\vee$ meet the properties required by Definition \ref{def:latticeposet}.
\begin{itemize}
\item\textbf{Supremum}:
\begin{itemize}
\item\textbf{$\phi\subseteq\phi\vee\psi$}: \\
$\phi\subseteq\phi\cup(\psi\circ\phi)^\ast\implies\phi\subseteq\phi\vee\psi$.
\item\textbf{$\psi\subseteq\phi\vee\psi$}: \\
$\psi=\psi\circ\{(x,x)|x\in S\}\subseteq\psi\circ\phi\subseteq(\psi\circ\phi)^\ast\subseteq\phi\cup(\psi\circ\phi)^\ast\implies\psi\subseteq\phi\vee\psi$.
\item\textbf{$\phi\subseteq\varphi,~\psi\subseteq\varphi\implies\phi\vee\psi\subseteq\varphi$}: \\
$\phi\vee\psi=\phi\cup(\psi\circ\phi)^\ast\subseteq\varphi\cup(\varphi\circ\varphi)^\ast=\varphi\implies\phi\vee\psi\subseteq\varphi$.
\end{itemize}
\item\textbf{Infimum}:
\begin{itemize}
\item\textbf{$\phi\wedge\psi\subseteq\phi$}: \\
$\phi\wedge\psi=\phi\cap\psi\subseteq\phi$.
\item\textbf{$\phi\wedge\psi\subseteq\psi$}: \\
$\phi\wedge\psi=\phi\cap\psi\subseteq\psi\implies\phi\wedge\psi\subseteq\psi$.
\item\textbf{$\varphi\subseteq\phi,~\varphi\subseteq\psi\implies\varphi\subseteq\phi\wedge\psi$}: \\
$\varphi\subseteq\phi,~\varphi\subseteq\psi\implies\varphi\subseteq\phi\cap\psi=\phi\wedge\psi$.
\end{itemize}
\end{itemize}
Therefore, $\langle E(S),\subseteq\rangle$ is a lattice with $\wedge,\vee$ defined in Eq. \ref{eq:latticewedge}, \ref{eq:latticevee}.
\end{proof}
\begin{corollary}
Given a trajectory set $H$ of any POMDP, denote $E(H)$ the set of all equivalence relations on $H$, then $\langle E(H),\subseteq\rangle$ is a lattice.
\end{corollary}
\subsection{Stationarity and Consistency}
\label{app:statandcons}
\textbf{Stationarity}: We use the term (intra-trajectory) \textbf{stationarity} to describe the invariance of the transition dynamics across steps within the same trajectory.
It means that during any single run of the environment, transitions between any two steps follow the same logic.
The formal definition is as follows:
\begin{definition}[\textbf{$p$-Stationary}]
For any POMDP $\pomdp=\pomdpdef$ with trajectory set $H$, let $p:H\to p(H)$ be a map.
The POMDP $\pomdp$ is $p$-\textbf{stationary} if, for all $h_{t_1},h_{t_2}\in H$ where $h_{t_1}\preceq h_{t_2}$ and $p(h_{t_1})=p(h_{t_2})$, the transition probabilities satisfy $T_{t_1}(z,r|z_{0:t_1},a_{0:t_1})=T_{t_2}(z,r|z_{0:t_2},a_{0:t_2})$ for all $z\in Z,~r\in\real$.
\end{definition}
To facilitate subsequent discussions, we first present the concept of terminal trajectory.
\begin{definition}[\textbf{Terminal Trajectory}]
For any POMDP $\pomdp$ with trajectory set $H$, a trajectory $h\in H$ is a \textbf{terminal trajectory} if there exists no $h^\prime$ such that $h\preceq h^\prime$. The set of all terminal trajectories is denoted by $\partial H$.
\end{definition}
\textbf{Note}: Based on the above definitions, the following conclusions can be drawn.
\begin{itemize}
    \item Since $h_{t_1}\preceq h_{t_2}$, this implies there exists $h_T\in\partial H$ such that $h_{t_1}\preceq h_{t_2}\preceq h_T$.
In this sense, $h_{t_1},h_{t_2}$ can be said to lie within the same trajectory $h_{T}$.
    \item In this paper, the adopted $p$ satisfies $p:(z_{0:t},a_{0:t},r_{0:t-1})\mapsto(z_{t-k+1,t},a_{t-k+1,t})$ which means it considers the most recent $k$ observation-action pairs.
    From this, $\simeq_k$-invariant can be derived.
    \item In this paper, we \textbf{break stationarity} by the following method: trajectories are divided into several stages by time intervals of equal length, with the same transition dynamics applied within each stage but potentially different ones across stages.
    \item It should be noted that this is not the only way to break stationarity; we merely chose one that is relatively simple and easy to implement.
    The switching of transition dynamics within a trajectory can depend not only on the numerical value of the time step but also on special events occurring in the trajectory, which is quite common in real-world problems.
\end{itemize}

\textbf{Consistency}: 
We use the term (inter-trajectory) \textbf{consistency} to describe the invariance of the transition dynamics among different trajectories.
It means that all trajectories follow the same transition logic at moments equivalent under some criterion $q$.
The formal definition is as follows:
\begin{definition}[\textbf{$q$-Consistent}]
For any POMDP $\pomdp=\pomdpdef$ with trajectory set $H$, let $q:\nat\to q(\nat)$ be a map.
The POMDP $\pomdp$ is $q$-\textbf{consistent} if, for all $h_{t_1},h^\prime_{t_2}\in H$ where $q(t_1)=q(t_2)$, the transition probabilities satisfy $T_{t_1}(z,r|z_{0:t_1},a_{0:t_1})=T_{t_2}(z,r|z^\prime_{0:t_2},a^\prime_{0:t_2})$ for all $z\in Z,~r\in\real$.
\end{definition}

\textbf{Note}: Based on the above definitions, the following conclusions can be drawn.
\begin{itemize}
    \item Since the above definition does not restrict whether $h_{t_1}$ and $h^\prime_{t_2}$ form a prefix relation, this captures the meaning of "different trajectories" in the description of the consistency concept.
    \item In this paper, the adopted $q$ satisfies $q:t\mapsto\lfloor\frac{t}{n}\rfloor$ which means it treats moments within the same interval of length $n$ as equivalent.
    From this, $\simeq^n$-invariant can be derived.
    \item In this paper, we \textbf{break consistency} using the following approach: the observation space is divided into several intervals of equal length, and the category of a trajectory is determined by the interval containing its initial observation.
    Trajectories within the same category share the same transition dynamics at equivalent time steps, while no such constraint applies to trajectories across different categories.
    \item  It should be noted that this is not the only way to break consistency: we merely chose one that is relatively simple and easy to implement.
    We partition trajectories into disjoint sets according to the interval containing their initial observations, which satisfy
    $\forall h_{t_1},h_{t_2}^\prime\in H,~h_{t_1}\sim h_{t_2}^\prime\iff\exists h\in H,~h\preceq h_{t_1},~h\preceq h_{t_2}^\prime$.
    Thus, this partitioning method enables the agent to possess sufficient information at any moment in the trajectory to determine the category of the current trajectory.
    However, this is not mandatory, as other ways of partitioning the trajectory set can also be adopted, in which case this property is violated.
\end{itemize}
\subsection{Visualization of (Non-)Stationarity and (Non-)Consistency}
$k$ is the order of autoregressive (AR) process, meaning $z_{t+1}$ depends on the most recent $k$ observations $z_{t-k+1:t}$.

$n$ denotes the division of the time axis into segments of length $n$, within each segment, the same (stationary) or different (non-stationary) AR coefficients can be used for transition calculations.

$m$ represents the number of intervals in the uniform partition of the observation set. 
A trajectory’s category is determined by the interval containing its initial observation. Across different categories, the AR coefficients used for transition calculations within the same time segment can be the same (consistent) or different (non-consistent).

\label{app:exptabfig}
\begin{figure}[h]
\centering
\includegraphics[width=0.4\columnwidth]{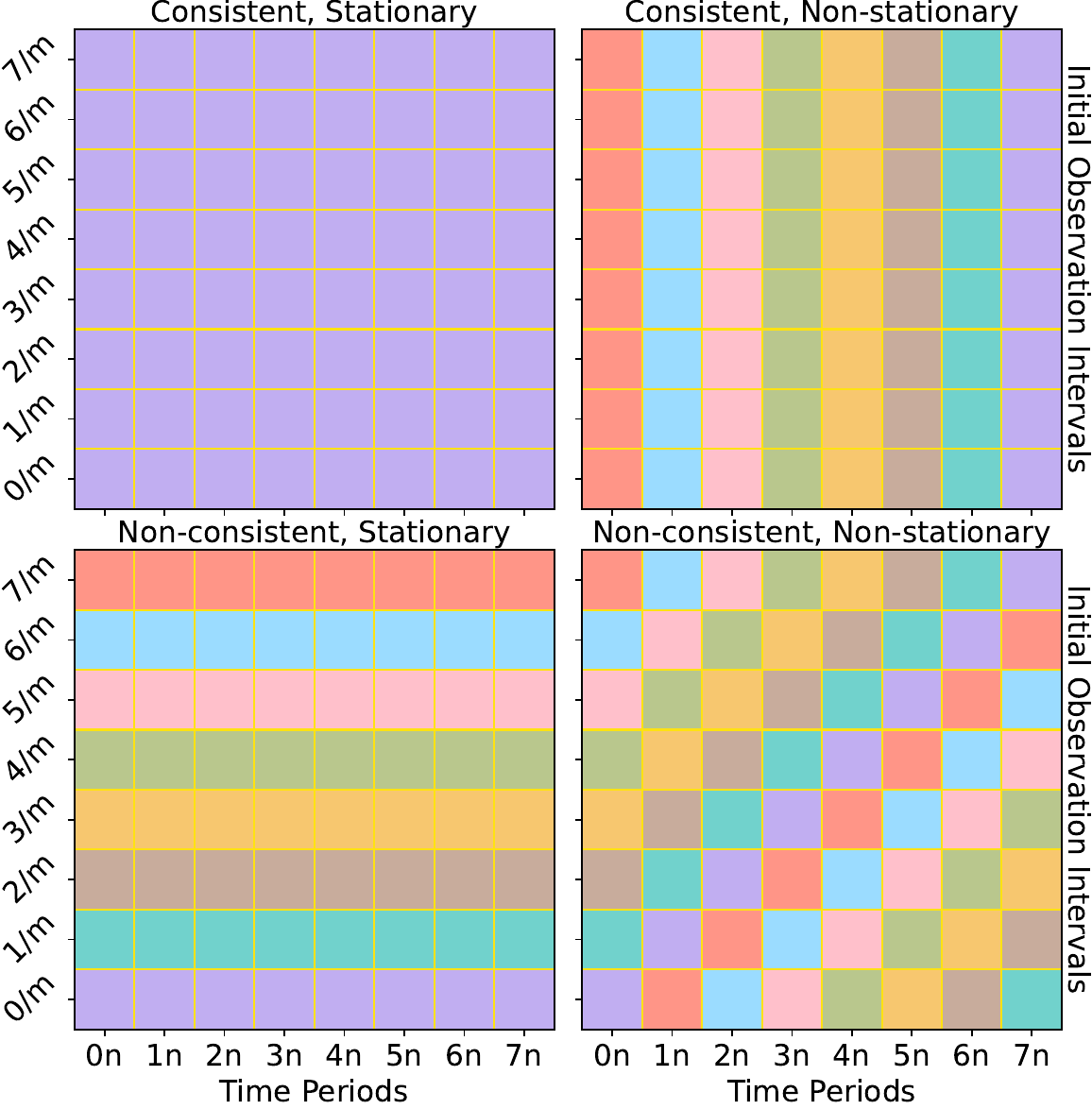}
\caption{Corresponding to Tab. \ref{tab:pomdps}, the x-axis shows $n$-length time periods in trajectories, and the y-axis shows distinct trajectories differentiated by the $\frac{1}{m}$-length intervals that their first observations belong to. Different colors correspond to different transitional and reward dynamics.}
\end{figure}

\begin{table}[h]
\centering
\begin{tabular}{|c|c|c|}
\hline
\textbf{Invariance} & \textbf{Stationary} & \textbf{Non-stationary} \\
\hline
\textbf{Consistent} & $\simeq_k$ & $\simeq_k\wedge\simeq^n$ \\
\hline
\textbf{Non-consistent} & $\simeq_k\wedge\simeq^{\lfloor mz_0\rfloor}$ & $\simeq_k\wedge\simeq^n\wedge\simeq^{\lfloor mz_0\rfloor}$ \\
\hline
\end{tabular}
\caption{Example of different transition invariances. $\simeq_k$ denotes a relation where trajectories $h_{t_1},h_{t_2}^\prime$ are equivalent if $(z,a)_{t_1-k+1:t_1}=(z^\prime,a^\prime)_{t_2-k+1:t_2}$; $\simeq^n$ where equivalent if $\lfloor t_1/n\rfloor=\lfloor t_2/n\rfloor$; $\simeq^{\lfloor mz_0\rfloor}$ where equivalent if $\exists i,z_0,z_0^\prime\in[\frac{i}{m},\frac{i+1}{m})$.}
\label{tab:pomdps}
\end{table}
\begin{table}[H]
\centering
\begin{tabular}{|c|c|c|}
\hline
\textbf{Invariance} & \textbf{Stationary} & \textbf{Non-stationary} \\
\hline
\textbf{Consistent} & $z_{t+1}=w_0z_t+w_1z_{t-1}$ & $z_{t+1}=w_0^tz_t+w_1^tz_{t-1}$ \\
\hline
\textbf{Non-consistent} & $z_{t+1}=\begin{cases}w_0z_t+w_1z_{t-1}~(z_0\le\frac{1}{2}) \\ w_0^\prime z_t+w_1^\prime z_{t-1}~(z_0>\frac{1}{2})\end{cases}$ & $z_{t+1}=\begin{cases}w_0^tz_t+w_1^tz_{t-1}~(z_0\le\frac{1}{2}) \\ w_0^{\prime t}z_t+w_1^{\prime t}z_{t-1}~(z_0>\frac{1}{2})\end{cases}$ \\
\hline
\end{tabular}
\caption{Examples of observation generation AR processes under different stationarity and consistency conditions, with $k=2,n=1,m=2$.}
\end{table}

\section{POMDP Environments Constructed from Scratch}
\label{app:constructedfromscratch}

\begin{figure}[h]
\centering
\includegraphics[width=0.8\columnwidth]{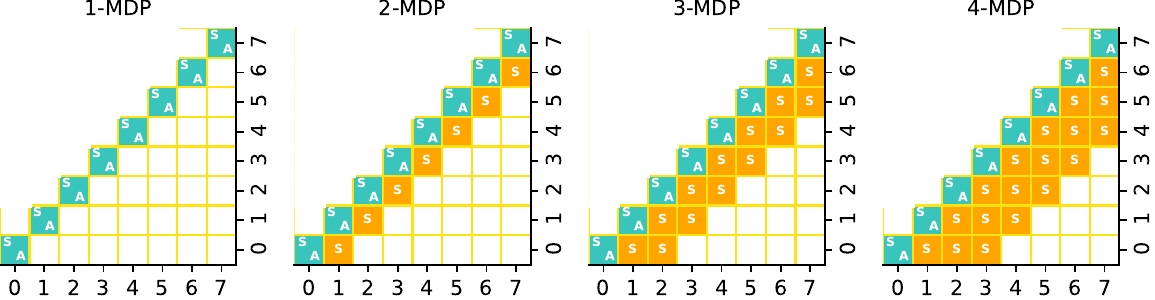}
\caption{MDS visualizations of high-order MDPs, with order increasing from $1$ to $4$.}\label{fig:mdpmds}
\end{figure}
\begin{figure}[h]
\centering
\includegraphics[width=0.5\columnwidth]{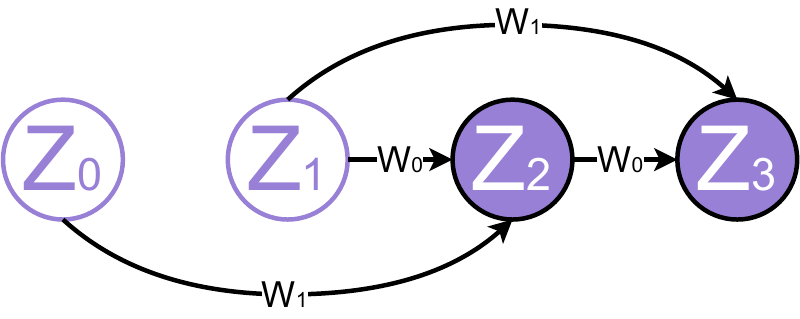}
\caption{An example AR process with coefficients $w_0,w_1$. $z_0,~z_1$ is sampled independently from $U(0,1)$.}\label{fig:arpvis}
\end{figure}
\noindent{\textbf{Order Range}: $k\in\{0,1,2,3,4,5,6,7\}$.}\\
\textbf{Observation Space}: $[0,1)$.\\
\textbf{Action Space}: $\{0,1,2,3,4,5,6,7\}$.\\
\textbf{Trajectory Length}: $64$.\\
\textbf{Initialization}: $z_t\sim U(0,1)~(0\le t<k)$.\\
\textbf{Coefficients}: $w_{0:7}=(8,9,10,11,12,13,14,15)/8\subset[1,2)$.\\
\textbf{Reward}: $r_t=\mathbf{1}_{a_t=\left\lfloor8\left(\sum_{i=0}^{k-1}w_iz_{t-i}\right)\mod 1\right\rfloor}~(t\ge k-1),~r_t=\mathbf{1}_{a_t=\left\lfloor8\left(\sum_{i=0}^{t}w_iz_{t-i}\right)\mod 1\right\rfloor}~(t<k-1)$.

Note that $z_{t+1}=\left(\sum_{i=0}^{k-1}w_iz_{t-i}\right)\mod 1$ when $t\ge k-1$, otherwise sampled from $U(0,1)$.
\subsection{\textbf{AllEqOneLinProc$(k)$}}
\textbf{Coefficients}: $1_{0:7}$.\\
\textbf{Transition}: $z_{t+1}=\left(\sum_{i=0}^{k-1}z_{t-i}\right)\mod 1$.\\
\textbf{Description}: AR process with all coefficients set to $1$, used to test the memory model's ability to retain observations within a finite time window of length $k$. Each observation in the window has the same weight.
\subsection{\textbf{AllEqOneQuadProc$(k)$}}
\textbf{Coefficients}: $1_{0:7}$.\\
\textbf{Transition}: $z_{t+1}=\left(\sum_{i=0}^{k-1}z_{t-i}^2\right)\mod 1$.\\
\textbf{Description}: Similar to \textbf{AllEqOneLinProcEnv}, but uses squared AR process.
\subsection{\textbf{AllEqLinProc$(k)$}}
\textbf{Coefficients}: $w_{0:7}$.\\
\textbf{Transition}: $z_{t+1}=\left(\sum_{i=0}^{k-1}w_iz_{t-i}\right)\mod 1$.\\
\textbf{Description}: AR process with varying coefficients, used to test the performance of the memory model in consistent and stationary environments of different orders.
All coefficients are no less than $1$, which is to prevent $z_t$ from converging to $0$.
\subsection{\textbf{TimeEqLinProc$(k)$}}
\textbf{Coefficients}: Define operator $\sigma: w_{0:T}\mapsto(w_T,w_{0:T-1})$.

Then
$w_{0:7,0:7}\coloneqq (\sigma^0(w_{0:7}),\sigma^1(w_{0:7}),\sigma^2(w_{0:7}),\dots,\sigma^7(w_{0:7}))=
\begin{pmatrix}
    8  &9 &10 &11 &12 &13 &14 &15 \\
    15 &8 &9 &10 &11 &12 &13 &14 \\ 
    14 &15 &8 &9 &10 &11 &12 &13 \\ 
    13 &14 &15 &8 &9 &10 &11 &12 \\ 
    12 &13 &14 &15 &8 &9 &10 &11 \\ 
    11 &12 &13 &14 &15 &8 &9 &10 \\ 
    10 &11 &12 &13 &14 &15 &8 &9 \\ 
    9 &10 &11 &12 &13 &14 &15 &8 \\ 
\end{pmatrix}/8$\\
\textbf{Transition}: $z_{t+1}=\left(\sum_{i=0}^{k-1}w_{\lfloor t/8\rfloor,i}z_{t-i}\right)\mod 1$.\\
\textbf{Description}: AR process with varying coefficients, used to test the performance of the memory model in consistent but non-stationary environments of different orders.
Non-stationarity means that the transition functions differ between various moments within the same trajectory, i.e., the adopted AR coefficients are different.
To avoid either insignificant differences in transition relationships between time segments due to an excessive number of distinct coefficients or excessive instability of the process caused by overly large coefficients, a circulant matrix construction is employed.
\subsection{\textbf{TrajEqLinProc$(k)$}}
\textbf{Coefficients}:

$w^\prime_{0:7,0:7}\coloneqq (\sigma^0(w_{0:7}),\sigma^{-1}(w_{0:7}),\sigma^{-2}(w_{0:7}),\dots,\sigma^{-7}(w_{0:7}))=\begin{pmatrix}
    8  &9 &10 &11 &12 &13 &14 &15 \\
    9 &10 &11 &12 &13 &14 &15 &8 \\
    10 &11 &12 &13 &14 &15 &8 &9 \\
    11 &12 &13 &14 &15 &8 &9 &10 \\
    12 &13 &14 &15 &8 &9 &10 &11 \\
    13 &14 &15 &8 &9 &10 &11 &12 \\
    14 &15 &8 &9 &10 &11 &12 &13 \\
    15 &8 &9 &10 &11 &12 &13 &14 \\
\end{pmatrix}/8$\\
\textbf{Transition}: $z_{t+1}=\left(\sum_{i=0}^{k-1}w^\prime_{\lfloor 8z_0\rfloor,i}z_{t-i}\right)\mod 1$.\\
\textbf{Description}: AR process with varying coefficients, used to test the performance of the memory models in non-consistent but stationary environments of different orders.
Non-consistency means that in the same environment, an agent may face multiple possible tasks, each with its unique transition regulation; i.e., the AR coefficients used differ across sequences whose initial observations belong to different intervals.
To avoid either insignificant differences in transition relationships between different trajectory categories due to an excessive number of distinct coefficients or excessive instability of the process caused by overly large coefficients, a circulant matrix construction is employed.
\subsection{\textbf{NoEqLinProc$(k)$}}
\textbf{Coefficients}: $w_{i,j,0:7}\coloneqq\sigma^{i-j}(w_{0:7}),~i,j\in\{0,1,2,3,4,5,6,7\}$.\\
\textbf{Transition}:
$z_{t+1}=\left(\sum_{i=0}^{k-1}w_{\lfloor t/8\rfloor,\lfloor8z_0\rfloor,i}z_{t-i}\right)\mod 1$.\\
\textbf{Description}: AR process with varying coefficients, used to test the performance of the memory models in non-consistent and non-stationary environments of different orders.
A circulant tensor is used based on the considerations outlined in the previous subsections.
\section{State Aggregation}
\label{app:stateagg}
\subsection{History Aggregator for State (HAS) and Optimality Preserving Property}
\begin{definition}[\textbf{Latest Extraction Operators}]
The latest state, action, and reward extraction operators at time $t$ are:
\begin{itemize}
    \item $\mathscr{L}_{S,t}: h_t\mapsto s_t~~(\forall t\in\nat,~\forall h_t\in H_t)$;
    \item $\mathscr{L}_{A,t}: h_{t+1}\mapsto a_{t}~~(\forall t\in\nat,~\forall h_{t+1}\in H_{t+1})$;
    \item $\mathscr{L}_{\real,t}: h_{t+1}\mapsto r_{t}~~(\forall t\in\nat,~\forall h_{t+1}\in H_{t+1})$.
\end{itemize}
\end{definition}
\begin{definition}[\textbf{History Aggregator for State (HAS)}]
For any MDP $\mdp=\mdpdef$, an HAS $\ha{S}{}\coloneqq\hadef{S}{}$ is a series of maps in which $\mathscr{A}_{S,t}:H_t\to \mathscr{A}(H),~(\forall t\in\nat)$ where $\mathscr{A}(H)$ represents a set called the target state set of $\mathscr{A}_S$.
\end{definition}
\textbf{Note}:
For convenience, we assume $\ha{S,t}{}: S^{t+1}\to Z$, where $Z$ denotes the observation space of the target POMDP.
This means actions and rewards do not participate in constructing the current observation.

\begin{definition}[\textbf{Reversibility of HAS on MDP}]
An HAS $\ha{S}{}=\hadef{S}{}$ on MDP $\mdp=\mdpdef$ is reversible iff there exists a series of maps, denoted as $\ha{S}{\ast}\coloneqq\hadef{S}{\ast}$, which satisfy:
$$\mathscr{A}_{S,t}^\ast(\{\mathscr{A}_{S,\tau}(h_\tau)\}_{\tau=0}^t)=s_t,~~(\forall t\in\nat)$$
where $h_\tau\in H_\tau, ~h_0\prec h_1\prec\cdots\prec h_t, ~s_t=\mathscr{L}_{S,t}(h_t)$.
\end{definition}

The reversibility condition of the HAS ensures that a decision-making algorithm can reconstruct the original MDP state from the POMDP trajectory.
This ensures that a reversible HAS does not alter the original belief state (i.e., the state of the MDP itself).
Thus, the reversible HAS ensures that the inherent difficulty remains unchanged when wrapping an MDP into an POMDP. Here, inherent difficulty refers to the difficulty posed to the RL algorithm itself, excluding the difficulty of extracting useful information from POMDP history to construct MDP states (i.e., the difficulty for the memory model).
\begin{definition}[\textbf{HAS Induced POMDP}]
For any MDP $\mdp=\mdpdef$ with trajectory set $H=\bigsqcup_{t=0}^\infty H_t$ and an HAS $\ha{S}{}=\hadef{S}{}$, where $\ha{S,t}{}:H_t\to\ha{}{}(H)$.
Applying $\ha{S}{}$ to $\mdp$ yields an POMDP $\pomdp=\pomdpdefp$, where: $\rho_0^\prime=\rho_0\circ\ha{S,0}{\ast}$, $Z=\ha{}{}(H)$, $T^\prime_t(z_{t+1},r_t|z_{0:t},a_{0:t})=T_t(s_{t+1},r_t|s_t,a_t),~z_\tau=\ha{S,\tau}{}(s_{0:\tau}),~(0\le \tau\le t+1)$, which is called the POMDP induced from $\mdp$ by HAS $\ha{S}{}$.
\end{definition}
\begin{definition}[\textbf{HAS Induced Policy}]
Given an MDP $\mdp=\mdpdef$, let $\pomdp=\pomdpdefp$ be the POMDP induced by HAS $\ha{S}{}$.
For any policy $\pi:S\to \Delta_A$ of $\mdp$, the policy $\tilde{\pi}_t(a_t|z_{0:t})\coloneqq \pi_t(a_t|s_t)$ where $z_\tau=\ha{S,\tau}{}(s_{0:\tau}),~(0\le\tau\le t)$, is called the policy induced from $\pi$ by HAS $\ha{S}{}$.
\end{definition}
\textbf{Note}: It can be stated that for any policy $\tilde{\pi}$ of POMDP $\pomdp$, there exists a policy $\pi$ of MDP $\mdp$ satisfying $\pi_t(a_t|s_t)=\tilde{\pi}_t(a_t|z_{0:t}),~s_\tau=\ha{S,\tau}{\ast}(z_{0:\tau}),~(0\le\tau\le t)$, such that $\tilde{\pi}$ is induced from $\pi$ by HAS $\ha{S}{}$.

\begin{definition}[\textbf{MDP Optimal Policy}]
For any MDP $\mdp=\mdpdef$, $\pi^\ast\coloneqq\pi^\ast_{0:\infty}$ is said to be an optimal policy of $\mdp$ if $\forall s_0\in S,~\pi^\ast\in\arg\max_{\pi} V^{\pi}(s_0)$, where $V^\pi$ is the value function: $V^\pi(s_0)=\ex_{h_{0:\infty}\sim(\pi_{0:\infty},T_{0:\infty}|s_0)}\sum_{t=0}^\infty r_{t}$ and $h_{0:\infty}=(s_{0:\infty},a_{0:\infty},r_{0:\infty})$.
\end{definition}
\textbf{Note}: The discount factor $\gamma$ can be absorbed into the time-dependent transition representation $T_t$ or the MDP can be restricted to a finite length before falling into an absorbing state with subsequent rewards fixed at zero. Both approaches ensure the convergence of the value function for the MDP defined above.

\begin{definition}[\textbf{POMDP Value Function}]
Given POMDP $\pomdp=\pomdpdef$ with trajectory set $H$ and a policy $\pi=\pi_{0:\infty}$, the \textbf{value function} is 
$V^\pi(z_0)\coloneqq\ex_{h_{0:\infty}\sim(\pi_{0:\infty},T_{0:\infty}|z_0)}\sum_{t=0}^\infty r_t$
where $h_{0,\infty}=(z_{0:\infty},a_{0:\infty},r_{0:\infty})\in H$, $(\pi_{0:\infty},T_{0:\infty}|z_0)$ represents the probability of generating a terminal trajectory that starts from $z_0$ through $T_{0,\infty}$ and policy $\pi_{0:\infty}$.
\end{definition}
\begin{definition}[\textbf{POMDP Optimal Policy}]
For any POMDP $\pomdp=\pomdpdef$, $\pi^\ast\coloneqq\pi^\ast_{0:\infty}$ is said to be an \textbf{optimal policy} of $\pomdp$ if $\forall z_0\in Z,~\pi^\ast\in\arg\max_\pi V^\pi(z_0)$.
\end{definition}

\begin{theorem}[\textbf{Optimality Preserving Property of Reversible HAS}]
Given an MDP $\mdp=\mdpdef$, let $\pomdp=\pomdpdefp$ be the POMDP induced by a reversible HAS $\ha{S}{}$.
For any optimal policy $\pi^\ast$ of $\mdp$, the policy $\tilde{\pi}^\ast$ induced from $\pi^\ast$ by $\ha{S}{}$ is an optimal policy of POMDP $\pomdp$.
\end{theorem}
\begin{proof}
Since $\pi^\ast$ is an optimal policy of $\mdp$, then \begin{equation}\pi^\ast\in\arg\max_{\pi}V^\pi(s_0)=\arg\max_{\pi}\ex_{h_{0:\infty}\sim(\pi_{0:\infty},T_{0:\infty}|s_0)}\sum_{t=0}^\infty r_{t}\label{eq:mdpoptcond}\end{equation}
Since $\tilde{\pi}^\ast$ is the policy induced from $\pi^\ast$ by reversible HAS $\ha{S}{}$, then 
\begin{equation}\tilde{\pi}_t^\ast(a_t|z_{0:t})=\pi_t^\ast(a_t|s_t)\end{equation}
\begin{equation}z_\tau=\ha{S,\tau}{}(s_{0:\tau}),~s_\tau=\ha{S,\tau}{\ast}(z_{0:\tau})\label{eq:hasbij}\end{equation}
Now the target is proving 
\begin{equation}\tilde{\pi}^\ast\in\arg\max V^{\tilde{\pi}}(z_0)=\arg\max_{\tilde{\pi}}\ex_{h^\prime_{0:\infty}\sim(\tilde{\pi}_{0:\infty},T^\prime_{0:\infty}|z_0)}\sum_{t=0}^\infty r_{t}\end{equation}
Assume there is a policy $\tilde{\pi}$ such that $\exists z_0,~V^{\tilde{\pi}}(z_0)>V^{\tilde{\pi}^\ast}(z_0)$, which is
\begin{equation}\ex_{h^\prime_{0:\infty}\sim(\tilde{\pi}_{0:\infty},T^\prime_{0:\infty}|z_0)}\sum_{t=0}^\infty r_{t}>\ex_{h^\prime_{0:\infty}\sim(\tilde{\pi}^\ast_{0:\infty},T^\prime_{0:\infty}|z_0)}\sum_{t=0}^\infty r_{t}\label{eq:contracond}\end{equation}
By Eq. \ref{eq:hasbij}, there is a bijection $\alpha$ between $H$ (trajectory set of $\mdp$) and $H^\prime$ (trajectory set of $\pomdp$):
\begin{equation}\alpha:(s_{0:t},a_{0:t},r_{0:t})\mapsto(\ha{S,t}{}(s_{0:t}),a_{0:t},r_{0:t}),~\alpha^{-1}:(z_{0:t},a_{0:t},r_{0:t})\mapsto(\ha{S,t}{\ast}(z_{0:t}),a_{0:t},r_{0:t})\end{equation}
Denote the MDP policy from which $\tilde{\pi}$ is induced as $\pi$, we can rewrite Eq. \ref{eq:contracond}:
\begin{equation}\ex_{h_{0:\infty}\sim(\pi_{0:\infty},T_{0:\infty}|s_0)}\sum_{t=0}^\infty r_{t}>\ex_{h_{0:\infty}\sim(\pi^\ast_{0:\infty},T_{0:\infty}|s_0)}\sum_{t=0}^\infty r_{t}\end{equation}
Therefore $\exists s_0\in S,~V^\pi(s_0)>V^{\pi^\ast}(s_0)$, contradicts Eq. \ref{eq:mdpoptcond}. \\
Thus, assumption Eq. \ref{eq:contracond} can not be true, which means $\forall z_0\in Z,~\tilde{\pi}^\ast\in\arg\max_{\tilde{\pi}} V^{\tilde{\pi}}(z_0)$, $\tilde{\pi}^\ast$ is an optimal policy of POMDP $\pomdp$.
\end{proof}
\textbf{Note}: The proof of this property is not provided in the HAS paper \href{https://arxiv.org/abs/2506.24026}{(Wang et al., 2025)}.
\subsection{Convolution Based Reversible HAS}
\label{subsec:convhas}
\textbf{General Form}
\begin{equation}z_t=\ha{S,t}{}(s_{0:t})\coloneqq w_{0:t}\ast s_{0:t}=\sum_{i=0}^t w_i s_{t-i}\end{equation}
\begin{equation}
\begin{pmatrix}
z_0\\
z_1\\
z_2\\
\dots\\
z_t
\end{pmatrix} = 
\begin{pmatrix}
w_0 & 0 & 0 & \dots & 0\\
w_1 & w_0 & 0 & \dots & 0\\
w_2 & w_1 & w_0 & \dots & 0 \\
\dots & \dots & \dots & \dots & \dots \\
w_t & w_{t-1} & w_{t-2} & \dots & w_0 \\
\end{pmatrix} \cdot
\begin{pmatrix}
s_0 \\
s_1 \\
s_2 \\
\dots \\
s_t
\end{pmatrix}
\end{equation}
\begin{equation}
z_{0:t}=W_t s_{0:t}
\end{equation}
Therefore, when $w_0\neq 0$,
\begin{equation}
s_{0:t}=W_t^{-1}z_{0:t}
\end{equation}
Thus, using the concept of POMDP MDS, $W_t^{-1}$ determines all observations $z_t$ required to decode the current state $s_t$.
\begin{example}
Given $w_{0:\infty}=(1,w,0^\ast)$, then
\begin{equation}
W_4=
\begin{pmatrix}
1 & 0 & 0 & 0\\
w & 1 & 0 & 0\\
0 & w & 1 & 0\\
0 & 0 & w & 1
\end{pmatrix};~W_4^{-1}=
\begin{pmatrix}
1 & 0 & 0 & 0\\
(-w)^1 & 1 & 0 & 0\\
(-w)^2 & (-w)^1 & 1 & 0\\
(-w)^3 & (-w)^2 & (-w)^1 & 1
\end{pmatrix}
\end{equation}
\end{example}
This example serves as the basis for the visualization in Fig. \ref{fig:aggmds}.
When $w$ is positive, the coefficients of each observation alternate in sign; when $w$ is negative, all coefficients of the observations are positive.

\begin{figure}[h]
\centering
\includegraphics[width=0.5\columnwidth]{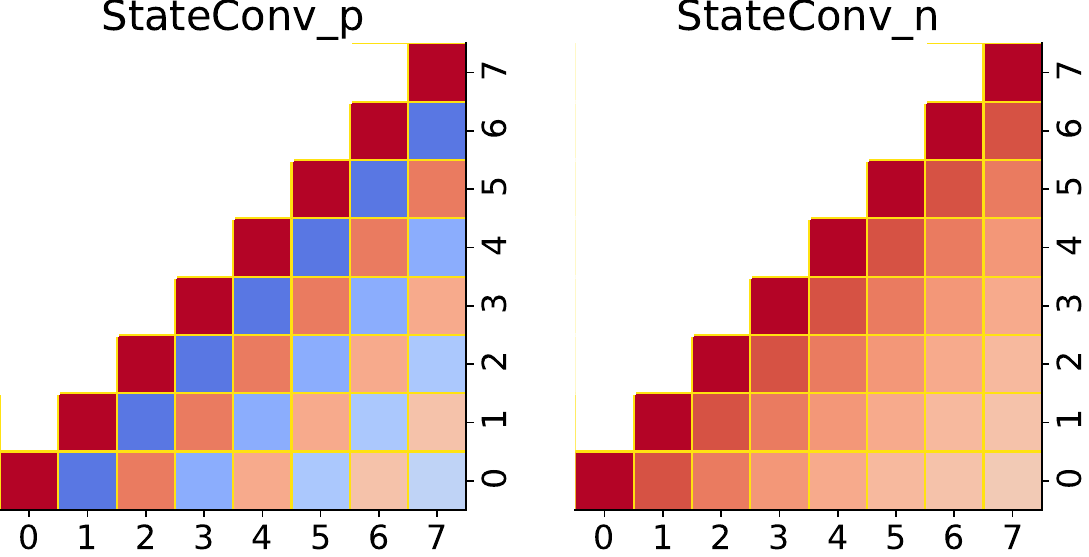}
\caption{Visualization of MDSs of convolution.
Left figure coefficients: $1,w$; Right figure coefficients: $1,-w$, $(w>0)$. 
In $W^{-1}_t$, positive elements are red, negative blue, and darker colors indicate larger absolute values.}
\label{fig:aggmds}
\end{figure}

\subsection{Generalized Implementation}
\subsubsection{Generalized Convolution}
In fact, besides using the addition and scalar multiplication of vectors in Euclidean space, there are other ways to implement these two operations of convolution, as long as they still form a linear space.
For example, given any Abelian group operator $+: S\times S\to S$, a simple way is to pick an arbitrary bijection $\eta:S\to S$ and define the addition operator to be $\oplus: (s,s^\prime)\mapsto\eta^{-1}(\eta(s)+\eta(s^\prime))$, the scalar multiplication operator to be $\odot:(w,s)\mapsto\eta^{-1}(w\cdot\eta(s))$.
It suffices to verify that such a definition still satisfies the requirements of a linear space. 

\subsubsection{Discrete Observation Space}
If the state space is discrete, there are two approaches to aggregate states for constructing an POMDP: either continuousizing the states and processing them in the same way as for continuous state spaces, or preserving the discreteness of the observation space.

The continuousization method is simple to implement. It merely requires mapping discrete states to integer codes and then treating them as real numbers.

Preserving discreteness, however, does not allow such an operation. It is necessary to define addition and scalar multiplication operations on discrete sets (e.g., $\{0,1,2,\dots,N-1\}$).
In this case, modular-$N$ addition and modular-$N$ scalar multiplication can be used to replace the corresponding operations in the continuous state space, and the invertibility condition of $W_t$ becomes that $w_0$ is coprime with $N$.
Since all structures of Abelian groups on finite sets are known, they are isomorphic to the direct sum of certain cyclic groups.
Even non-Abelian groups have complete classification theorems.
Coupled with the arbitrary choice of permutations on the state set, this provides a great variety of options for constructing convolution-based HAS.
\section{Reward Redistribution}
\label{app:rewardredist}
\subsection{Reward Redistribution and Optimality Preserving Property}
\label{subsec:rewardredist}
\begin{definition}[\textbf{Return-Equivalent POMDP}]
Two POMDP $\pomdp=\pomdpdef$ and $\pomdp^\prime=\pomdpdefp$ are return-equivalent if $Z=Z^\prime,~\rho_0=\rho_0^\prime,~A=A^\prime,~\sum_{r\in\real}T_t(z,r|z_{0:t},a_{0:t})=\sum_{r\in\real}T_t^\prime(z,r|z_{0:t},a_{0:t})$ and for all $z_0\in Z$ and any policy $\pi=\pi_{0,\infty}$, $V_{\pomdp}^\pi(z_0)=V^\pi_{\pomdp^\prime}(z_0)$.
\end{definition}
\textbf{Note}: For two return-equivalent POMDPs, their observation transition dynamics are identical, while their reward mechanisms are not entirely the same: although the rewards of each trajectory have different temporal distributions, their total amounts remain the same.
\begin{definition}[\textbf{Reward Redistribution}]
A reward redistribution is a map $\mathscr{R}:\pomdp\mapsto\pomdp^\prime$ where $\pomdp,\pomdp^\prime$ are return-equivalent POMDPs.
\end{definition}
\begin{theorem}[\textbf{Optimality Preserving Property of Reward Redistribution}]
For any reward redistribution $\mathscr{R}$ and any POMDP $\pomdp$, $\pomdp=\pomdpdef,~\pomdp^\prime=\mathscr{R}(\pomdp)=\pomdpdefp$ have the same set of optimal policies.
\end{theorem}
\begin{proof}
For any optimal policy $\pi$ of $\pomdp$, $\forall z_0\in Z$,
\begin{equation}\pi^\ast\in\arg\max_{\pi}V^\pi_{\pomdp}(z_0)=\arg\max_\pi\ex_{h_{0:\infty}\sim(\pi_{0:\infty},T_{0:\infty}|z_0)}\sum_{t=0}^\infty r_t\end{equation}
Since $\pomdp,\pomdp^\prime$ are return-equivalent, 
\begin{equation}\ex_{h_{0:\infty}\sim(\pi_{0:\infty},T_{0:\infty}|z_0)}\sum_{t=0}^\infty r_t=\ex_{h^\prime_{0:\infty}\sim(\pi_{0:\infty},T^\prime_{0:\infty}|z_0)}\sum_{t=0}^\infty r_t^\prime\end{equation}
where $h_{0:\infty}=(z_{0:\infty},a_{0:\infty},r_{0\infty}),~h^\prime_{0:\infty}=(z_{0:\infty},a_{0:\infty},r^\prime_{0\infty})$. Therefore, \begin{equation}\pi^\ast\in\arg\max_{\pi}\ex_{h^\prime_{0:\infty}\sim(\pi_{0:\infty},T^\prime_{0:\infty}|z_0)}\sum_{t=0}^\infty r_t^\prime=\arg\max_\pi V^\ast_{\pomdp^\prime}(z_0)\end{equation}
$\pi^\ast$ is an optimal policy of $\pomdp^\prime$.\\
By symmetry of return-equivalence relation of $\pomdp,\pomdp^\prime$, for any optimal policy $\tilde{\pi}^\ast$ of $\pomdp^\prime$, $\tilde{\pi}^\ast$ is also an optimal policy of $\pomdp$.
\end{proof}
\subsection{Reward Delay}
\label{subsec:rewdelay}
\begin{definition}[\textbf{Reward Delay}]
Reward delay $\mathscr{D}_k$ is a reward redistribution that maps POMDP $\pomdp=\pomdpdef$ (with trajectory set $H$) to $\pomdp^\prime=\langle\rho_0,Z,A,\{T_t^\prime\}_{t=0}^\infty\rangle$ (with trajectory set $H^\prime$) such that for any trajectory $h_{0:t}=(z_{0:t},a_{0:t},r_{0:t})\in H$, its corresponding trajectory $h^\prime_{0:t}=(z_{0:t},a_{0:t},r^\prime_{0:t})$ in $H^\prime$ satisfies (if $0\le T\le\infty$ is the max timestep)
$$r^\prime_{0:k-1}=0,~r^\prime_{k:T-2}=r_{0:T-k-2},~r^\prime_{T-1}=\sum_{i=0}^{k-1}r_{T-i-1}$$
The transition dynamics $T^\prime_t$ satisfy
$$T_t^\prime(z_{t+1},0|z_{0:t},a_{0:t})=\sum_{r_t\in\real} T_t(z_{t+1},r_t|z_{0:t},a_{0:t}),~(t<k)$$
$$T_t^\prime(z_{t+1},r^\prime_t|z_{0:t},a_{0:t})=\mathbf{1}_{r_{t-k}=r^\prime_t}\sum_{r_t\in\real}T_t(z_{t+1},r_t|z_{0:t},a_{0:t}),~(k\le t\le T-2)$$
$$T_{T-1}^\prime(z_T,r^\prime_{T-1}|z_{0:T-1},a_{0:T-1})=\mathbf{1}_{\sum_{i=0}^{k-1}r_{T-i-1}=r_{T-1}^\prime}\sum_{r_{T-1}\in\real}T_{T-1}(z_{T},r_{T-1}|z_{0:T-1},a_{0:T-1})$$
\end{definition}

\textbf{Note}: It should be noted that delaying all rewards of each time step to appear after a certain number of moments is not the only way to implement delayed rewards. Rewards can also be delayed in a certain proportion, where both the proportion and the delay time can be variable. It is obvious that reward delay $\mathscr{D}_k$ is a reward redistribution.

If there is a discount factor $\gamma$, the delayed reward should be 
$$r^\prime_{0:k-1}=0,~r^\prime_{k:T-2}=\frac{r_{0:T-k-2}}{\gamma^k},~r^\prime_{T-1}=\sum_{i=0}^{k-1}\frac{r_{T-1-i}}{\gamma^i}$$

\section{POMDP Environments Derived from Existing}
\label{app:derivedfromexisting}
\subsection{\textbf{StateConv$(k)$}}
\begin{figure}[h]
\centering
\includegraphics[width=0.5\columnwidth]{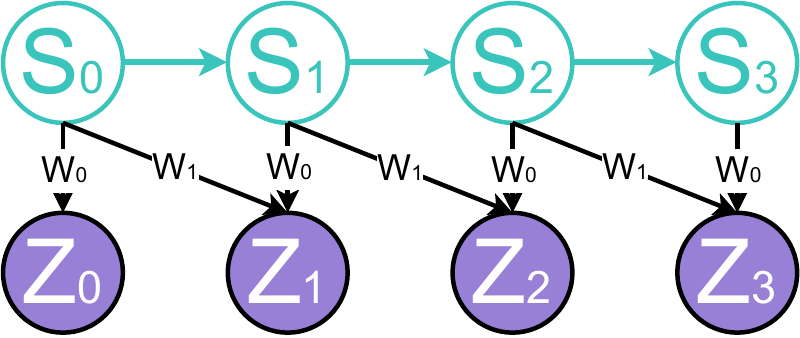}
\caption{A convolution-based HAS (with coefficients $w_0,w_1$) for aggregating MDP states into POMDP observations.} 
\label{fig:hasvis}
\end{figure}
\noindent{\textbf{Coefficients}:
$w_0^k=1,~w_!^k=1-\frac{1}{2^k},~(k\le 4),~w_1^5=1$.}\\
\textbf{Transition}:
$z_t=w_0^ks_t+w_1^ks_{t-1},~(t>0)$, $z_0=w_0^ks_0$.\\
\textbf{Description}:
Corresponding to Sec. \ref{subsec:convhas}, \textbf{StateConv$(k)$} is a series of environment wrappers. By wrapping a given MDP environment, it constructs the observation sequence of the POMDP by convolving the states of the MDP, while keeping the reward mechanism unchanged.
As $k$ increases, the current observation contains more information about historical states, which requires the memory model to remember more historical state information to decode the current Markov state, thus making the problem more difficult for the memory model.
This wrapper can be applied to various tasks with continuous observation spaces, provided that the states support scalar multiplication and addition operations.
\subsection{\textbf{RewardDelay$(k)$}}
\begin{figure}[h]
\centering
\includegraphics[width=0.8\columnwidth]{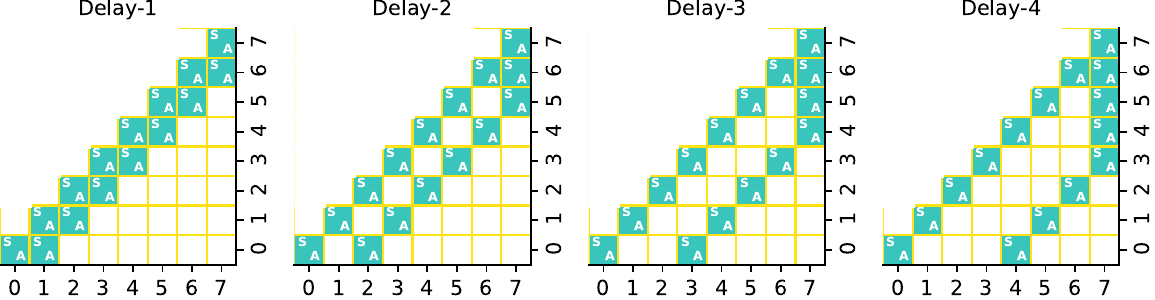}
\caption{MDS of reward delay. The next state depends on the current state and action, while the reward depends on states and actions from $k$ steps prior, illustrating the composition of the final reward as a discounted sum of the past $k$ rewards.}
\label{fig:mdpvis}
\end{figure}

\noindent{\textbf{Transition}: The reward appearing at moment $t$ in the given environment is delayed to moment $t+8k$, with the rewards at times $0$ to $8k-1$ set to zero. Meanwhile, the rewards in the $k$ moments before termination are all moved to the termination time. The delay is compensated according to the power of the discount.}\\
\textbf{Description}:
Corresponding to Sec. \ref{subsec:rewdelay}, \textbf{RewardDelay$(k)$} delays rewards while ensuring that the optimal policy remains unchanged, still being a Markov policy in the original MDP.
In this case, for agents using a memory model, the difficulty does not lie in remembering historical observations, constructing belief states, and making better decisions, but rather in learning a policy that relies solely on the current observation.
This series of wrappers tests the ability of the memory model to ignore irrelevant information during the learning process.
As $k$ increases, the number of steps by which rewards are delayed increases; since rewards are no longer generated immediately, the amount of irrelevant information that the memory needs to ignore also grows, which increases the fitting difficulty for the memory model.

\section{Experiments Settings}
\label{app:exp}
\subsection{Settings}
\textbf{CPU}: Intel(R) Xeon(R) Gold 6348 CPU @ 2.60GHz; 28 Cores, 112 Threads. \\
\textbf{GPU}: NVIDIA GeForce RTX 3080, Driver Version: 550.127.05, CUDA Version: 12.4; 8 GPUs. \\
\textbf{Memory}: 512GB. \\
\textbf{OS}: 24.04.1 LTS (Noble Numbat). \\
\textbf{Packages}: python[3.10.0]; numpy[1.25.0]; torch[2.7.1+cu126]; ray[2.46.0]; gymasium[1.0.0]; scipy[1.15.3]; popgym.
\subsection{Configs and Hyperparameters}
\textbf{Repeated Runs}: $10$ times for each experiment. \\
\textbf{Models}: 2-layer MLP with LeakyReLU activations, for observation encoder, policy head, and value head.\\
\textbf{Memory Models}: The standard baseline implementations from the POPGym library were used, with positional encoding removed to adapt to variable episode lengths.\\
\textbf{Hidden Size for Linear Layers and Memory}: 64.\\
\textbf{Backup through Time Truncation Length}: 128 for \textbf{StateConv} series and 64 for \textbf{RewardWhenInside} and \textbf{LinProc} series.\\
\textbf{Algorithm}: PPO using Python ray-rllib, modified from the implementation in POPGym. \\
Since DQN does not adapt well to the POPGym and RLlib frameworks, we implement it ourselves. \\
\textbf{Discount Factor $\gamma$}: $0$ for \textbf{LinProc} series and \textbf{RewardWhenInside} wrapped with \textbf{StateConv} series (Since maximizing the reward at each individual step in these environments ensures the maximization of the total reward over the entire trajectory, $\gamma$ is set to $0$), $0.9$ for other environments.\\